\documentclass[journal]{IEEEtran}

\IEEEoverridecommandlockouts                              
\usepackage{pifont}
\usepackage{stfloats}
\usepackage{breqn}
\usepackage{cuted}
\usepackage{multirow}
\usepackage{booktabs}
\usepackage{threeparttable}
\usepackage{subcaption}
\usepackage{algorithm}
\usepackage{algpseudocode}
\usepackage{lineno,hyperref}
\usepackage{amsmath,graphicx}
\usepackage{amsthm}

\usepackage{amssymb}
\usepackage{booktabs}
\usepackage{cite}
\usepackage{tablefootnote}
\usepackage{xcolor}
\usepackage{mathtools}
\usepackage{cases}
\usepackage{caption}  
\usepackage[final]{changes}
\usepackage{amsthm}
\newtheorem{theorem}{Theorem}[section]

\newtheorem{lemma}[theorem]{Lemma}

\newtheorem{corollary}[theorem]{Corollary}
\newtheorem{definition}[theorem]{Definition}
\newtheorem{assumption}[theorem]{Assumption}
\newtheorem{remark}[theorem]{Remark}

\newcommand\bA{\mathbf{A}}
\newcommand\bB{\mathbf{B}}
\newcommand\bC{\mathbf{C}}
\newcommand\bG{\mathbf{G}}

\newcommand\bE{\mathbb{E}}

\newcommand\teta{\tilde{\eta}}

\newcommand\tr{\tilde{r}}
\newcommand\bxi{\boldsymbol{\xi}}

\newcommand\dist{\operatorname{dist}}
\newcommand\reg{\operatorname{Regret}_d}
\newcommand\regs{\operatorname{Regret}_s}

\definecolor{purple}{rgb}{0.6, 0.2, 0.8}
\newcommand{\red}[1]{{\color{red}#1}}

\usepackage{multibib}
\definecolor{UniBlue}{RGB}{83,121,170}
\definecolor{DarkGray}{RGB}{90,90,90}
\definecolor{LightGray}{RGB}{150,150,150}
\definecolor{oldTextGreen}{RGB}{115,155,15}
\definecolor{teal}{RGB}{100, 200,10}
\definecolor{oldOcean}{RGB}{23,142,189}
\definecolor{Ocean}{RGB}{30,106,181}
\definecolor{BG}{RGB}{215,215,215}
\definecolor{darkred}{RGB}{204,41,0}

\usepackage{hyperref}

\allowdisplaybreaks
\fontsize{8}{9}\selectfont

\title{
Locally Differentially Private Online Federated Learning With Correlated Noise
}

\author{Jiaojiao Zhang, Linglingzhi Zhu, Dominik Fay and Mikael Johansson
\thanks{This work is supported in part by the funding from Digital Futures and VR under the contract 2019-05319. Parts of the material in this paper have been published at the 63rd IEEE Conference on Decision and Control.}
\thanks{Jiaojiao Zhang, Dominik Fay, and Mikael Johansson are with the Division of Decision and Control Systems, School of Electrical Engineering and Computer
Science, KTH Royal Institute of Technology, SE-100 44 Stockholm, Sweden. 
{\tt\small \{jiaoz,dominikf,mikaelj\}@kth.se  }}%
\thanks{Linglingzhi Zhu is with the H. Milton Stewart School of Industrial and Systems Engineering, Georgia Institute of Technology, Atlanta, Georgia, USA.
{\tt\small llzzhu@gatech.edu}}%
}

\begin{document}

\maketitle

\begin{abstract}
We introduce a locally differentially private (LDP) algorithm for online federated learning that employs temporally correlated noise to 
improve utility while preserving privacy. To address challenges posed by the correlated noise and local updates with streaming non-IID data, we develop a perturbed iterate analysis \replaced{that controls}{to control} the impact of the noise on the utility. Moreover, we demonstrate how the drift errors from local updates can be effectively managed \replaced{for several classes of nonconvex loss functions}{under a class of nonconvex conditions}. Subject to an $(\epsilon, \delta)$-LDP budget, we establish a dynamic regret bound \replaced{that quantifies}{, quantifying} the impact of key parameters and the intensity of changes in \replaced{the dynamic environment}{dynamic environments} on the learning performance. Numerical experiments confirm the efficacy of the proposed algorithm.
\end{abstract}
\section{Introduction}
In this paper, we focus on online federated learning (OFL)~\cite{cdc2021online,wang2023linear,liu2023differentially}, a framework that combines the principles of federated learning (FL) and online learning (OL) to address the challenges of real-time data processing across distributed data resources. In OFL, a central \textit{server} coordinates multiple \textit{learners}, each interacting with streaming \textit{clients} as they arrive sequentially. The client data is used collaboratively to improve the utility of all learners~\cite{kairouz2021advances, tsp-5,tsp-6}; see Fig.~\ref{fig-ofl}. 
\begin{figure}[htbp]
\centering
\includegraphics[width=7.5cm]{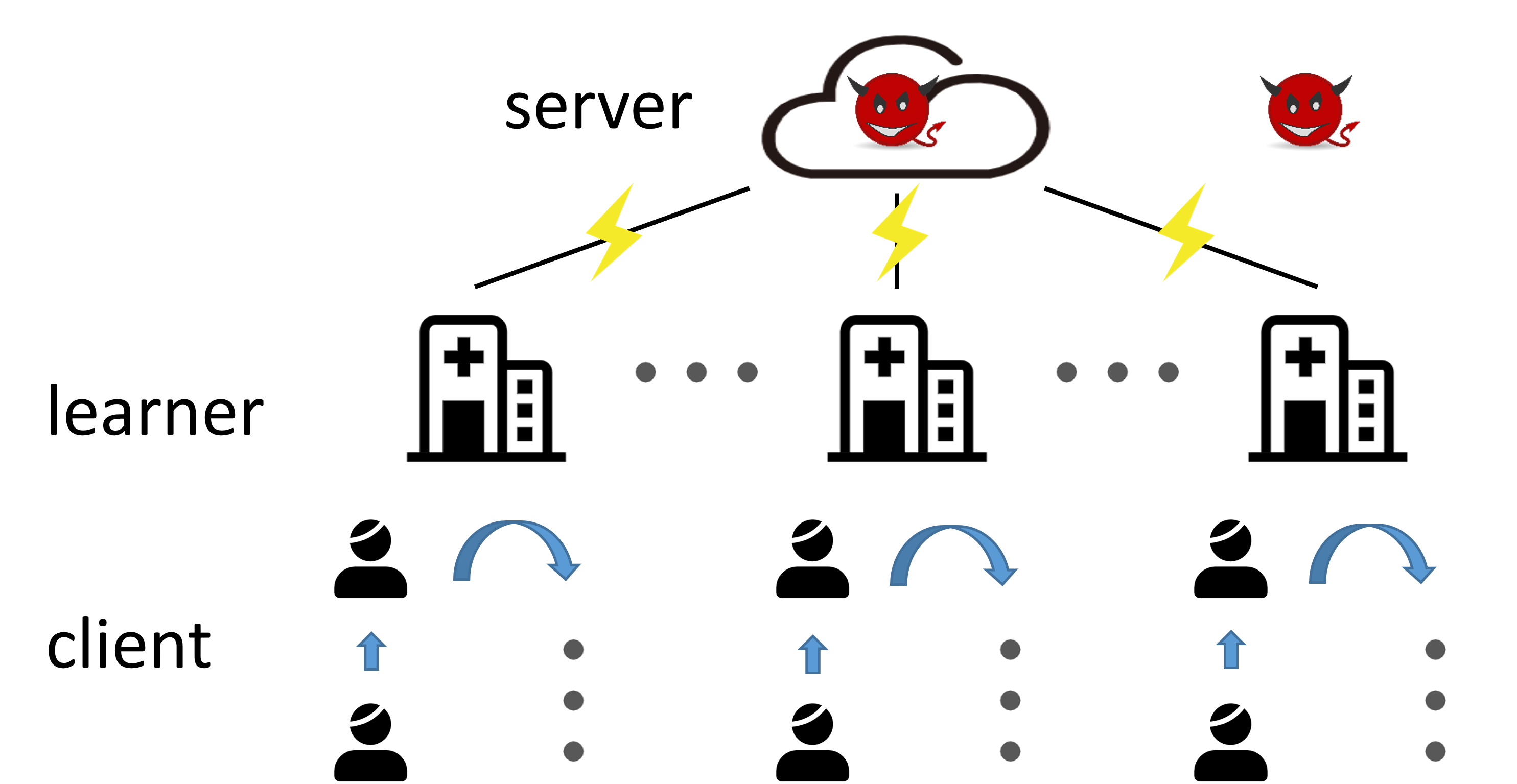}
\caption{OFL framework}
\label{fig-ofl}
\end{figure}

\replaced{
Traditional FL operates in an offline setting, where data is stored on learners and can be sampled IID (e.g., through random mini-batches at each iteration) from a fixed distribution, allowing for repeated use of the same data. However, in practical applications, data often arrives in a streaming fashion, making offline FL insufficient. This online setting introduces two main challenges: (i) Storing data is inefficient in terms of space and raises privacy concerns. Avoiding data storage optimizes resource usage and reduces the risk of data leakage and unauthorized access, but it also makes data reuse difficult. (ii) Streaming data that arrives at different time steps is typically non-IID, even for the same learner. Considering the potentially substantial differences among clients associated with different learners, data across learners can also exhibit non-IID characteristics, even at the same time step~\cite{cdc2021online}.}{
Traditional FL operates offline, assuming that data can be stored on learners and sampled IID from a fixed distribution, allowing repeated use of the data. However, in practical applications, we often face streaming data, making offline FL insufficient: (i) Since data storage is inefficient in terms of space or undesirable due to privacy concerns, not storing data can lead to more efficient use of resources and reduce the risk of data leakage and unauthorized access. However, it typically makes data reuse difficult.  (ii) Streaming data that arrive at different time steps is typically non-IID, even on the same learner. Considering the possibly substantial differences among clients associated with different learners, the data across learners also exhibits non-IID characteristics, even in the same time step~\cite{cdc2021online}.}
Due to time-varying data streams, updating and releasing the model with new data can enhance model freshness.  This capability of providing continuously improved services is crucial in applications like \deleted{healthcare,} recommendation systems, predictive maintenance, and anomaly detection. These motivate us to study online FL.

A significant concern of federated learning
is the risk of privacy leakage. Clients in the online learning process need assurance that their sensitive private data is not exposed to others~\cite{kairouz2021practical,denisov2022improved}. Differential privacy (DP), which typically involves adding noise to sensitive information to guarantee the indistinguishability of outputs~\cite{dwork-29, tsp-9, tsp-11}, is widely recognized as a standard technique for preserving and quantifying privacy. 
Most research on DP federated learning adds privacy-preserving noise independently across iterations, but this noise reduces the utility significantly~\cite{liu2023differentially,bassily2014private}. Recently, some authors have proposed algorithms that use temporally correlated noise to enhance the privacy-utility trade-off in single-machine online learning~\cite{kairouz2021practical,denisov2022improved,koloskova2023convergence,koloskova2024gradient}. However, no theoretical guarantees have been developed for the privacy-utility trade-off when applying correlated noise in online federated learning scenarios. A key difference between online federated learning and single-machine online learning is the use of local updates to improve communication efficiency~\cite{liu2023differentially}. These local updates, combined with streaming non-IID data, make utility analysis more challenging, especially when introducing privacy protection through correlated noise.  

\subsection{Contributions}
Considering an honest but curious server and eavesdroppers, we propose an LDP algorithm that extends temporally correlated noise mechanisms, previously studied in single-machine settings, to OFL. 
Using a perturbed iterate technique, we analyze the combined effect of correlated noise, local updates, and streaming non-IID data. Specifically, we construct a virtual variable by subtracting the DP noise from the actual variable generated by our algorithm and use it as a tool to establish a dynamic regret bound for the released global model.  Subject to an $(\epsilon, \delta)$-LDP budget, we establish a dynamic regret bound over several classes of nonconvex loss functions that quantifies the impact of key parameters and the intensity of changes in the dynamic environment on the learning performance. Numerical experiments validate the efficacy of our algorithm.

\subsection{Related Work}\label{sec-related-work}

To the best of our knowledge, no existing work has developed theoretical guarantees for OFL with local updates,  correlated noise for privacy protection, and nonconvex loss functions. However, several papers have considered partial or related aspects of this problem. For a simple overview, we provide a comparison in Table~\ref{tab} and include a more comprehensive discussion of related work, covering many more papers, below.

\subsubsection{Correlated noise} The use of temporally correlated noise for privacy protection in single-machine online learning has recently been studied by multiple research groups~\cite{kairouz2021practical,denisov2022improved,koloskova2023convergence,koloskova2024gradient,jaja2023almost}. The proposed algorithms can be represented as a binary tree~\cite{dwork2010differential,jain2023price}, where the privacy analysis ensures that the release of the entire tree remains private. 
The study by Kairouz et al.~\cite{kairouz2021practical} utilized the binary mechanism to develop a differentially private variant of the Follow-The-Regularized-Leader (DP-FTRL) algorithm with a provable regret bound. The Google AI blog highlighted the use of DP-FTRL in their deployments~\cite{mcmahan2022federated}. 
In addition to the binary tree mechanism, the matrix factorization (MF) mechanism—originally developed for linear counting queries~\cite{li2015matrix}—can also be used to construct temporally correlated noise. In fact, the binary tree mechanism is a specific instance of the more general MF approach, whose additional flexibility can be used to improve the utility-privacy trade-off even further~\cite{li2015matrix,koloskova2023convergence,jaja2023almost}. For example, Denisov et al.~\cite{denisov2022improved} proposed an optimization formulation for the matrix factorization that they could solve using a fixed-point algorithm, and observed that MF-based stochastic gradient descent significantly improves the privacy-utility trade-off compared to a traditional binary tree mechanism. However,  these findings are primarily empirical. Henzinger et al.~\cite{jaja2023almost} proposed an MF mechanism based on Toeplitz matrices, whose elements can be explicitly iteratively solved. This approach not only improves the regret bound of the binary tree mechanism in~\cite{kairouz2021practical} by a constant factor but also provides a theoretical explanation for the empirical improvements observed by Denisov et al.~\cite{denisov2022improved}.

As highlighted in~\cite{dwork2010differential,denisov2022improved,jain2023price}, using correlated noise in online learning, unlike offline learning, requires consideration of adaptive continual release. Continual release refers to a privacy-preserving mechanism that handles both streaming inputs and outputs. In~\cite{denisov2022improved,jain2023price}, the inputs are streaming gradients computed from streaming raw data, while the outputs are the noisy versions of linear queries on these gradients. The privacy of the raw data must be preserved when all streaming outputs are continuously observed. Moreover, one should also consider adaptive inputs since the point at which the gradient is computed is related to previous outputs. Some studies have shown that both binary tree and MF mechanisms can handle adaptive continual release~\cite{denisov2022improved,jain2023price}. However, extending this approach to design LDP online FL with local updates using correlated noise is challenging and remains unexplored.

\subsubsection{Online distributed learning}

LDP with independent noise has been explored for decentralized online learning in~\cite{cdc2021online,li2018differentially,xiong2020privacy,cheng2023distributed,liu2023differentially}. This setting includes the server-learner scenario as a special case, albeit without multiple local updates. In particular, Liu et al.\cite{liu2023differentially} aggregate a mini-batch of gradients to perform a single local update, meaning the local model is updated only once per communication round. In contrast, our approach involves multiple local updates per communication round, similar to FL algorithms like FedAvg\cite{mcmahan2017communication}, where local models undergo several updates before communication. The work in~\cite{cdc2021online} also considers OFL with local updates but lacks DP protection. Both \cite{liu2023differentially} and \cite{cdc2021online} focus on static regret for convex OFL. By contrast, our algorithm introduces correlated noise and multiple local updates, requiring a different design and analyses to establish dynamic regret bounds for nonconvex problems.

\subsubsection{Dynamic regret for nonconvex problems} 
Even in a single-machine setting without privacy protection, establishing dynamic regret bounds for nonconvex problems requires new analytical techniques. On the one hand, compared with static regret bounds, dynamic regret bounds are stricter and more suited to scenarios with dynamic changes in the environment. However, achieving a sublinear dynamic regret is difficult even under strong convexity. Intuitively, when the environment changes rapidly, online learning faces greater challenges in achieving high utility. On the other hand, establishing a sublinear regret for nonconvex problems, even for static regret, presents significant challenges~\cite[Proposition 3]{suggala2020online}. 
The paper~\cite{suggala2020online} studies static regret for general nonconvex problems but requires an offline algorithm oracle to minimize the aggregated loss. Nonconvex online learning has been studied under special conditions on the loss functions. The work \cite{gao2018online} considers weakly pseudo convex objective functions and establishes dynamic regret bounds. The work \cite{zhang2017improved} considers semi-strongly convex objectives and improves the dynamic regret bound but needs data to be repeatedly used. By contrast, we consider the case when data is only used once and establish a dynamic regret bound for a class of nonconvex problems. Compared to~\cite{gao2018online,zhang2017improved}, novel analytical approaches are required to manage the correlated noise and local updates.

{\bf Notation.} Unless otherwise specified, all variables are $d$-dimensional row vectors. Accordingly,  loss functions map  $d$-dimensional row vectors to real numbers. 
The Frobenius norm of a matrix is denoted by $\|\cdot\|_F$, and the $\ell_2$-norm of a row vector is represented by $\|\cdot\|$. 
The notation $[n]$ refers to the set $\{1,\ldots, n\}$, and  $\operatorname{P}_{\mathcal{X}^{\star}}^x$ denotes the projection of $x$ onto the set $\mathcal{X}^{\star}$.
We use $\operatorname{Pr}$ to denote the probability of a random event and $\bE$ for the expectation. 
We define $\bA\in \mathbb{R}^{R\tau\times R\tau}$ as a lower triangle matrix with 1's on and below the diagonal and ${\bf I}$ as the identity matrix.
Given constants $R$, $\tau$, and $W$, we define matrices $\bG_i$, $\bB$, $\bC$, and $\bxi_i$ as 
\begin{equation*}
\underbrace{
\begin{bmatrix}
\nabla f_i^{0,0}\\
\ldots\\
\nabla f_i^{0,\tau-1}\\
\ldots\\
\nabla f_i^{R-1,0}\\
\ldots\\
\nabla f_i^{R-1,\tau-1}
\end{bmatrix}
}_{\bG_i\in\mathbb{R}^{R\tau\times d}}
\underbrace{
\begin{bmatrix}
 b^{0,0}\\
\ldots\\
b^{0,\tau-1}\\
\ldots\\
 b^{R-1,0}\\
\ldots\\
 b^{R-1,\tau-1}
\end{bmatrix}
}_{\bB\in \mathbb{R}^{R\tau\times W}}
\underbrace{
\begin{bmatrix}
 c^{0,0}\\
\ldots\\
c^{0,\tau-1}\\
\ldots\\
 c^{R-1,0}\\
\ldots\\
c^{R-1,\tau-1}
\end{bmatrix}
}_{\bC^T\in \mathbb{R}^{R\tau\times W}}
\underbrace{
\begin{bmatrix}
 \xi_i^1\\
\xi_i^2\\
\ldots\\
\xi_i^W
\end{bmatrix}
}_{\bxi_i\in \mathbb{R}^{W\times d}}.
\end{equation*} 
The notation $\bxi_i\sim \mathcal{N}(0, V_i^2)^{W\times d}$ indicates that all entries of $\bxi_i$ are independent and follow the Gaussian distribution $\mathcal{N}(0, V_i^2)$.
 

\section{Problem Formulation}\label{sec-setting}

{\bf Online federated learning.}  As shown in Fig.~\ref{fig-ofl}, our setting comprises one  server and $n$ learners, where each learner $i\in [n]$ interacts with streaming clients that arrive sequentially. We refer to the model parameters on the server as the global model and the models on the learners as local. The server's task is to coordinate all learners in online training of the global model, which is continuously released to the clients to provide instant service. To enhance communication efficiency, learner \(i\) performs \(\tau\) steps of local updates, each step using data from a different client, before sending the updated model to the server. To describe this intermittent communication, we define the entire time horizon as $\{0, 1, \ldots, \tau-1, \ldots, (R-1)\tau, \ldots, R\tau-1\}$, with communication occurring at time step $ \{r\tau: r\in[R]-1\}$. In this setup, there are $R$ communication rounds, each separated by $\tau$ steps.
The utility of the sequence of global models $\{x^r\}_r$ is measured by the dynamic regret
\begin{equation}\label{eq-dynm}
\reg:=\sum_{r=0}^{R-1}\sum_{t=0}^{\tau-1} \frac{1}{n}\sum_{i=1}^{n}(f_i^{r,t}(x^r)-(f^r)^{\star}). 
\end{equation}
Here, $f_i^{r,t}(x^r)$ is the loss incurred by the global model $x^r$ on data $D_i^{r,t}$ and  
$(f^r)^{\star}$ is a dynamic optimal loss defined by
\begin{equation*}
(f^r)^{\star}:=\min_{x}  f^r(x):=\frac{1}{n\tau}\sum_{i=1}^n\sum_{t=0}^{\tau-1} f_i^{r,t} (x).
\end{equation*}
The term \textit{dynamic} refers to a regret measure that compares the loss incurred by our algorithm to a sequence of time-varying optimal losses, as opposed to the commonly used \textit{static} regret
\begin{equation*}
\regs:=\sum_{r=0}^{R-1}\sum_{t=0}^{\tau-1} \frac{1}{n}\sum_{i=1}^{n}(f_i^{r,t}(x^r)-f^r(x^{\star})), 
\end{equation*}
where $x^{\star}$ represents an optimal model, in the optimal solution set $ \mathcal{X}^{\star}$  that minimizes the cumulative loss over entire data, 
\begin{equation*}
x^{\star}\in \mathcal{X}^{\star} := \operatorname{argmin}_x\sum_{r=0}^{R-1}\sum_{t=0}^{\tau-1} \frac{1}{n}\sum_{i=1}^{n}f_i^{r,t}(x).  
\end{equation*} 
{$\regs$, which compares an algorithm’s performance to a single, globally optimal model,
is reasonable when all data is available in advance.
In contrast, $\reg$ compares against a sequence of optimal models
and accounts for how the optimal solution may shift under changing conditions.
This is more stringent, but also more relevant in many OFL settings \cite{jiang2022distributed,eshraghi2022improving}.
For example, in disease prediction, the best predictor may vary with season , and in recommendation systems,
user preferences often evolve over time.}

In our paper, we aim to learn a series of models $\{x^r\}_r$ that minimizes $\reg$ while satisfying privacy constraints. 

{\bf Privacy threat model.}
We consider an honest-but-curious server and eavesdroppers capable of intercepting the communication between the server and the learners, as illustrated in Fig.~\ref{fig-ofl}. To protect privacy, each learner adds temporally correlated noise locally {at each local update} before transmitting information to the server. As a result, the noise across time is not independent. We aim to guarantee local differential privacy of each client’s data, even if the exchanged information is observed by attackers, i.e., the server and the eavesdroppers. When a client with data $D_i^{r,t}$ arrives, the learner obtains the client's data and calculates the gradient $\nabla f_i^{r,t}$ to update the local model once, and then discards the data without storing it. We assume that the original client data provided to the learner is not accessible to the attackers, as this process is not public.  Since our algorithm transmits local gradient information, we use the MF mechanism for each learner $i$ to add correlated noise to each local gradient.  
{Specifically, the privacy protection mechanism in our algorithm has streaming inputs \(\{\nabla f_i^{0,0}, \ldots, \nabla f_i^{r,t}\}\)  (\emph{i.e.}, all gradients processed by learner $i$ so far) and streaming outputs consisting of noisy prefix sums  \(\{\nabla f_i^{0,0},\nabla f_i^{0,0}+\nabla f_i^{0,1},\ldots, \nabla f_i^{0,0}+\cdots+\nabla f_i^{r,t}\}\), which refers to continual release.} In addition, inputs are adaptive, meaning that the next input {\(\nabla f_i^{r,t+1}\)} depends on previous outputs. This requires us to account for a more powerful attacker who can influence input selection; nevertheless, our algorithm remains LDP under adaptive continual release.

We quantify privacy leakage via LDP. We define the aggregated dataset of learner $i$ over the entire time horizon as $D_i = \{D_{i}^{r,t}: r\in[R]-1, t\in[\tau]-1\}$. LDP is used over \textit{neighboring} datasets $D_i$ and $D^{\prime}_i$ that differ by a single entry (for instance, replacing $D_i^{r,t}$ by ${D_i^{r,t}}^{\prime}$). We use the following LDP definition:
\begin{definition}
A randomized algorithm $\mathcal{M}$ satisfies $(\epsilon, \delta)$-LDP if for any pair of neighboring datasets $D_i$ and $D_i^{\prime}$, and for any set of outcomes $O$ within the output domain of $\mathcal{M}$,
\begin{equation*}
\operatorname{Pr}[\mathcal{M}(D_i) \in {O}] \leq e^\epsilon \operatorname{Pr}\left[\mathcal{M}\left(D_i^{\prime}\right) \in {O}\right]+\delta.    
\end{equation*}
The level of privacy protection is quantified by two parameters $(\epsilon,\delta)$ where smaller values indicate stronger protection.
\end{definition}

\section{Algorithm}
In this section, we present the proposed algorithm and a privacy-preserving mechanism that utilizes correlated noise via matrix factorization.

\subsection{Proposed Algorithm}
We propose a locally differential private OFL algorithm  
outlined in Algorithm~\ref{alg-fl}.  
Key features of our algorithm include the use of temporally correlated noise to protect privacy and the use 
of local updates to reduce communication frequency between the server and learners. Mathematically, the proposed algorithm can be re-written as the updates
\begin{equation}\label{eq-comp-alg}
\left\{
\begin{aligned}
z_{i}^{r,t+1}& = 
z_{i}^{r,t}-\eta \left({\nabla} f_i^{r,t}(z_i^{r,t}) + (b^{r,t}-b^{r,t-1}) \bxi_i\right),\\
z_{i}^{r,\tau}& = 
x^r-\eta \tau \left(g_i^r+\frac{1}{\tau}(b^{r,\tau-1}-b^{r-1,\tau-1})\bxi_i \right),\\
x^{r+1}& =x^r - \teta \left( g^r + \frac{1}{\tau}(b^{r,\tau-1}-b^{r-1,\tau-1})\bxi \right),
\end{aligned}
\right.
\end{equation}
where $g_i^r=\frac{1}{\tau} \sum_{t=0}^{\tau-1}\nabla f_i^{r,t}(z_i^{r,t})$, $g^r=\frac{1}{n}\sum_{i=1}^n g_i^r$, and $\bxi=\frac{1}{n} \sum_{i=1}^n \bxi_i$. 
{The equivalence of \eqref{eq-comp-alg} and Algorithm \ref{alg-fl} is derived in Supplementary. }

Inspired by the single-machine online learning, we implement Line 9 in Algorithm~\ref{alg-fl} using the MF mechanism. 
\begin{algorithm}[h]
\caption{Proposed Algorithm}
\label{alg-fl}
\begin{algorithmic}[1]
\State $\textbf{Input:}$ $R, \tau, \eta, \eta_g$, $\bxi_i$, $x^{0}$, {$b^{0,-1}=0$, $b^{r,-1}=b^{r-1,\tau-1}$} 
\State Set $\teta=\eta \eta_g \tau$
\For {$r = 0, 1, \ldots, R-1$ }
\State {\bf Learner $i$}
\State Receive $x^r$ from the server
\State Set $z_{i}^{r,0}=x^r$
\For {$ t= 0, 1, \ldots, \tau-1$ }
\State {Compute $\nabla f_i^{r,t}:=\nabla f_i^{r,t}(z_i^{r,t})$}
\State {Use MF to obtain $$S_{i}^{r,t}:=\nabla f_i^{0,0}+\cdots+\nabla f_i^{r,t}+b^{r,t}\bxi_i$$}
\State {Set $\hat{\nabla} f_i^{r,t}=S_{i}^{r,t}-S_{i}^{r,t-1}$}
\State { Update 
$z_{i}^{r,t+1}= 
z_{i}^{r,t}-\eta \hat{\nabla} f_i^{r,t}$ }
\EndFor
\State {Set $\hat{g}_i^r:=\frac{1}{\eta\tau}(x^r-z_i^{r,\tau}) =\frac{1}{\tau} \sum_{t=0}^{\tau-1}\hat{\nabla} f_i^{r,t}$  }
\State Transmit $\hat{g}_i^{r}$ to the server
\State {\bf Server}
\State Update  $x^{r+1}=x^r-\teta\frac{1}{n}\sum_{i=1}^n \hat{g}_{i}^{r}$ 
\EndFor
\State {\bf Output:} $\{x^1,\ldots, x^R\}$ 
\end{algorithmic}
\end{algorithm}

\subsection{Adding Correlated Noise via MF}

MF has recently been used to generate correlated noise to enhance utility and privacy of single-machine OL~\cite{kairouz2021practical,jaja2023almost}. These papers  assume $x^0 = 0$ and express the iterates of a gradient algorithm as $x^{r+1} = -\,\eta\sum_{\tilde{r}=0}^{r} g^{\tilde{r}}$, $\forall\,r\in [R]-1$, where $g^{\tilde{r}}$ is the gradient direction at iteration $\tilde{r}$. Consequently, the key DP objective is to estimate the prefix sums $\{\,g^0,\,g^0 + g^1,\,\ldots,\,g^0 + \cdots + g^{R-1}\}$ over the individual gradients.
Due to the distributed nature of FL and its use of local updates, this approach can not be applied directly to our setting. When learner $i$ updates its local model via $z_i^{r,\tau} = x^r - \eta\,\tau\,g_i^r$ (omitting noise for clarity), it begins from a global parameter $x^r$ that incorporates other learners’ updates, so $z_i^{r,\tau}$ cannot be viewed as a simple prefix sum of $g_i^r$.

Instead, a new approach is needed. In our design, we focus on the difference $\hat{\nabla}f_i^{r,t} = S_i^{r,t} - S_i^{r,t-1}$ 
(Line 10 in Algorithm \ref{alg-fl}) to enable correlated noise injection and preserve LDP.

In the following, we show that our algorithm can be interpreted as post-processing \cite{dwork-29} of {$\{\nabla f_i^{0,0}+b^{0,0}\bxi_i, \ldots, \nabla f_i^{0,0}+\cdots+\nabla f_i^{R-1,\tau-1}+b^{R-1,\tau-1}\bxi_i\}$. 
With $b^{-1,\tau-1}=b^{0,-1}=0$ and $x^0=0$, repeated application of the last step in~\eqref{eq-comp-alg} yields
\begin{equation*}
\begin{aligned}
x^{r}&=-\teta \left( \sum_{\tr=0}^{r-1} g^{\tr} +\frac{1}{\tau}b^{r-1,\tau-1}\bxi\right)\\
&=-\teta \frac{1}{n}\sum_{i=1}^{n}\left( \sum_{\tr=0}^{r-1} g_i^{\tr} +\frac{1}{\tau}b^{r-1,\tau-1}\bxi_i\right)\\
&=-\teta \frac{1}{n}\sum_{i=1}^n \frac{1}{\tau}\left( \nabla f_i^{0,0}+\cdots+\nabla f_i^{r-1,\tau-1} + b^{r-1,\tau-1}\bxi_i \right).    
\end{aligned} 
\end{equation*}}
With this equality, we observe that both the transmitted variables, $x^r$ and $\hat{g}_i^r$ in Lines 5 and 14 of Algorithm~\ref{alg-fl}, respectively, are post-processed versions of the noisy prefix sums.  It is, therefore, sufficient to release noisy prefix sums privately. To this end, we use MF. For mathematical clarity, we arrange the {$R\tau$} entries of {$\{\nabla f_i^{0,0}+b^{0,0}\bxi_i, \ldots, \nabla f_i^{0,0}+\cdots+\nabla f_i^{R-1,\tau-1}+b^{R-1,\tau-1}\bxi_i\}$} as the $R\tau$ rows of an {\(R\tau \times d\)} matrix, resulting in  
\begin{equation}\label{eq-3} 
\mathbf{A} \mathbf{G}_i + \mathbf{B} \bxi_i,  
\end{equation}
where $\mathbf{A}$ is a lower triangular matrix with all entries on and below the diagonal equal to $1$. Although each entry of $\bxi_i$ is independent, the multiplication by the matrix ${\bf B}$ introduces correlations among the rows of ${\bf B}\bxi_i$ which complicates the privacy analysis. A strategic approach to address this is to decompose the matrix ${\bf A}$ as ${\bf A} = {\bf BC}$, and use ${\bf B}$ to construct temporally correlated noise. By substituting ${\bf {\bf A}}={\bf BC}$ into~\eqref{eq-3} and factoring out ${\bf B}$, we have
\begin{equation}\label{eq-abc} 
\mathbf{B} (\bC\mathbf{G}_i + \bxi_i).  
\end{equation}
Here, the noise $\bxi_i$ with iid entries is added to ${\bC\bG_i}$. The privacy loss of~\eqref{eq-abc} can then be interpreted as the result of post-processing following a single application of the Gaussian mechanism~\cite{denisov2022improved}.

Below, we present three state-of-the-art methods for implementing MF: (i) the binary tree mechanism, (ii) solving MF with optimization techniques, and (iii) using Toeplitz matrices.
  
\begin{figure*}
\begin{equation}\label{eq-bc}
\underbrace{
\begin{bmatrix}
\nabla f_i^{0,0} \\
\nabla f_i^{0,0}+\nabla f_i^{0,1}\\
\nabla f_i^{0,0}+\nabla f_i^{0,1}+\nabla f_i^{1,0}\\
\nabla f_i^{0,0}+\nabla f_i^{0,1}+\nabla f_i^{1,0}+\nabla f_i^{1,1}\\
\end{bmatrix}
}_{\bA {\bG}_i}
+
\underbrace{\scalebox{1}{$
\begin{bmatrix}
\xi_i^1 \\
\xi_i^3\\
\xi_i^3+\xi_i^4\\
\xi_i^7\\
\end{bmatrix}$ }
 }_{{\bB \bxi}_i}
=\underbrace{
\scalebox{0.8}{$
\begin{bmatrix}
1&0&0&0&0&0&0\\ 
0&0&1&0&0&0&0\\  
0&0&1&1&0&0&0\\
0&0&0&0&0&0&1\\
\end{bmatrix} 
$}
}_{\bB}
\Bigg(
\underbrace{
\scalebox{0.7}{$
\begin{bmatrix}
1&0&0&0\\ 
0&1&0&0\\  
1&1&0&0\\ 
0&0&1&0\\ 
0&0&0&1\\  
0&0&1&1\\ 
1&1&1&1\\ 
\end{bmatrix}
$}
}_{\bC}
\underbrace{
\scalebox{1}{$\begin{bmatrix}
\nabla f_i^{0,0}\\
\nabla f_i^{0,1}\\
\nabla f_i^{1,0}\\
\nabla f_i^{1,1}
\end{bmatrix}$}}_{{\bG}_i}
+ 
\underbrace{
\scalebox{0.7}{$
\begin{bmatrix}
\xi_i^1\\
\xi_i^2\\
\xi_i^3\\
\xi_i^4\\
\xi_i^5\\
\xi_i^6\\
\xi_i^7\\
\end{bmatrix}$ }
}_{{\bxi}_i}\Bigg). 
\end{equation}
{\noindent}\rule[-12pt]{17.8cm}{0.1pt}
\end{figure*}

\begin{figure}
\centering
\includegraphics[width=8cm]{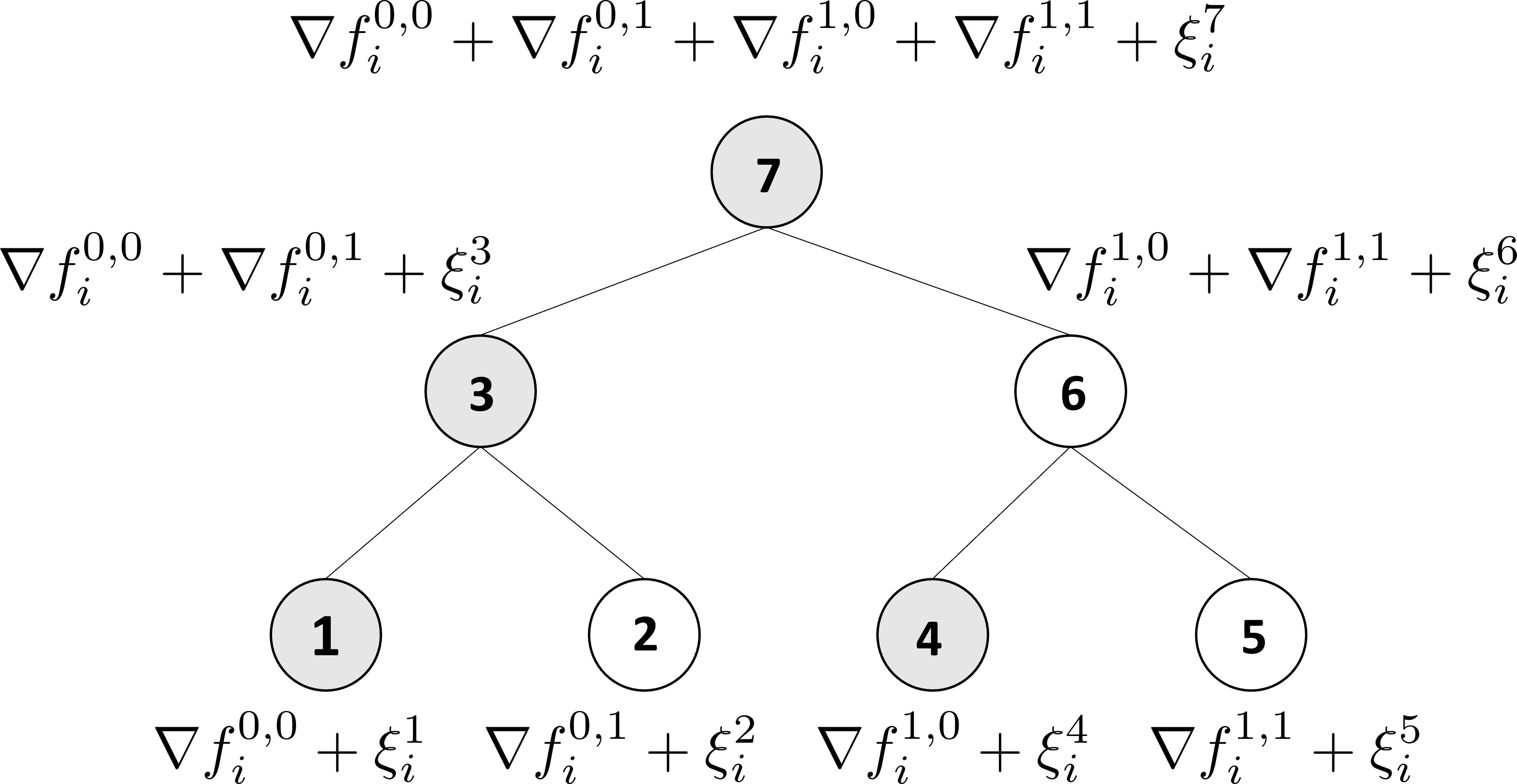}
\caption{Binary tree mechanism}
\label{fig:tree}
\end{figure}

\paragraph*{\bf Example (i)}
The binary tree mechanism releases differentially private prefix sums based on selected information computed hierarchically. In the binary tree, each leaf node stores an input value, while internal nodes store the sum of their left and right children. To ensure privacy, zero-mean Gaussian noise with variance $V_i^2$ is added when a node releases its stored value. The prefix sums are estimated from the outputs of a subset of the nodes. Fig.~\ref{fig:tree} illustrates the case of $R=\tau=2$ and $W=7$ nodes. Here, sequentially releasing  the 1st node, the 3rd node, the sum of the 3rd and 4th nodes, and the 7th node (shown as dark nodes), allows us to estimate the prefix sums. 

Although the noise added at each node is independent, the noise in the prefix sums will be correlated, as seen in (\ref{eq-bc}). This correlation can improve utility. The number of ones in each column \(c^{r,t}\) of \(\bC\) represents how many times the input \(\nabla f_i^{r,t}\) appears across all nodes, which is at most \(\log_2(R\tau) + 1\). Meanwhile, the number of ones in each row \(b^{r,t}\) of \(\bB\) corresponds to the number of dark nodes used to estimate the prefix sum \(\nabla f_i^{0,0} + \cdots + \nabla f_i^{r,t}\), which is at most \(\log_2(R\tau)\). This leads to the bounds:
\begin{equation}\label{eq-bc-logR}
\|c^{r,t}\|^2 \le \log_2 ({R\tau}) + 1,~ \|b^{r,t}\|^2 \le \log_2 ({R\tau}), ~\forall r,t.
\end{equation}

From \( \bC \), we can calculate the noise variance $V_i^2$ added to each node in the tree to satisfy a fixed privacy budget, while  \( \bB \) allows us to quantify the impact of noise on the utility.

\begin{remark}[Comparison with independent noise]
For ease of comparison, we also use a special tree to realize the addition of independent noise. This tree has only leaves, i.e., the height of the tree is 0. In this case,  $\bC=\bf I$ and $\bB=\bA$. When a leaf changes, it affects only one in the tree. However, since all leaves are dark nodes, we must sum $R\tau$ leaves to estimate $\nabla f_i^{0,0}+\cdots+\nabla f_i^{R-1,\tau-1}$. This means that $ \|c^{r,t}\|^2=1, \forall r,t,$ and $\|b^{R-1,\tau-1}\|^2=R\tau$. Intuitively, compared to correlated noise, using independent noise adds less variance ($1$ vs $\log_2 (R\tau)+1$) to each node in the tree but introduces more noise overall ($R\tau$ vs $\log_2 (R\tau)$), resulting in worse utility. We will formally prove the advantage of correlated noise over independent noise in Corollary~\ref{lem-u-IID}. 
\end{remark}

\paragraph*{\bf Example (ii)}
The binary tree is a special case of MF, which offers more flexibility and the possibility of optimizing the factors \( \bB \) and \( \bC \) to improve  performance~\cite{kairouz2021practical,denisov2022improved}.

For instance, the minimal \(\ell_2\)-norm solution for the linear equation \(\bA = \bB \bC\) is given by \( \bA \bC^{\dagger} \), where \( \bC^{\dagger} \) is the Moore-Penrose pseudo-inverse of \( \bC \). Denisov et al.~\cite{denisov2022improved} therefore proposed to construct  the matrix factors \( \bB, \bC \in \mathbb{R}^{R\tau \times R\tau} \) by solving the following optimization problem:
\begin{equation}\label{eq-ABC}
\min_{\bC \in \mathbb{R}^{R\tau \times R\tau} }~ \|\bA \bC^{\dagger}\|_F^2, \quad \operatorname{s.t.}\; \bC \in \mathbf{V}, ~ \max_{r,t}~ \|c^{r,t}\| = 1,
\end{equation}
where \(\mathbf{V}\) is a linear space of matrices. A fixed-point algorithm to solve~\eqref{eq-ABC} was given in~\cite[Theorem 3.2]{denisov2022improved}.
Note that the factorization only requires prior knowledge of $R$ and $\tau$, and can be calculated offline before the algorithm begins.

\paragraph*{\bf Example (iii)}
The optimization formulation~\eqref{eq-ABC} includes a constraint \(\max_{r,t} \|c^{r,t}\| = 1\) to limit the sensitivity and uses an objective function to minimize \(\|\bB\|_F^2\). 
Empirically, this leads to higher utility, but it is challenging to derive theoretical bounds on  $\|\bB\|_F^2$. To address this issue, \cite{jaja2023almost} proposed to use the following  Toeplitz matrix construction for  $\bB$ and $\bC$: 
\[
\bB = \bC = \begin{bmatrix}
h(0) & 0 & \cdots & 0\\
h(1) & h(0) & \cdots & 0\\
\vdots & \vdots & \ddots & \vdots\\
h(R\tau-2) & h(R\tau-3) & \cdots & 0\\
h(R\tau-1) & h(R\tau-2) & \cdots & h(0)\\
\end{bmatrix},
\]
where \(h(j) = \begin{cases}
1, & j=0,    \\
\left(1-\frac{1}{2j}\right) h(j-1),& j\ge 1. 
\end{cases}\)
Thus, both \(\bB\) and \(\bC\) are Toeplitz matrices with all diagonal entries equal to $1$.
Furthermore, \cite[Section 5.1]{jaja2023almost} proved that 
\begin{equation}\label{eq-exmp-iii}
\begin{aligned}
& \|c^{r,t}\|^2\le\|c^{0,0}\|^2  \le 1+\frac{1}{\pi} \ln \left( \frac{4R\tau}{5}\right),\\
&\|b^{r,t}\|^2\le \|b^{R-1,\tau-1}\|^2 \le  1+\frac{1}{\pi} \ln \left( \frac{4R\tau}{5}\right),~\forall r, t,
\end{aligned}    
\end{equation}
which is of similar order as the results of \eqref{eq-bc-logR} derived from the binary tree method.

As noted in Section~\ref{sec-related-work}, Denisov et al.~\cite{denisov2022improved} found that the MF mechanism in Example (ii) outperforms the binary-tree method in Example (i) experimentally. Henzinger et al.~\cite{jaja2023almost} provided a theoretical explanation for this, showing that the MF mechanism in Example (iii) achieves a constant improvement over Example (i). In this paper, we use Example (iii) to construct upper bounds for \(\bB\) and \(\bC\) in our analysis and compare Examples (i), (ii), and (iii) in our experiments.

\section{Analysis}
We will now derive a dynamic regret bound for Algorithm \ref{alg-fl} solving a class of nonconvex problems subject to $(\epsilon,\delta)$-LDP. 
\subsection{Preliminaries}
We impose the following assumptions on the loss functions.    
\begin{assumption}\label{asm-smooth}
Each loss function $f_i^{r,t}$ is  $L$-smooth, i.e.,
for any $x,y$, there exists a constant $L$ such that
\begin{equation*}
f_i^{r,t} (y)\leq f_i^{r,t} (x)+ \langle \nabla f_i^{r,t} (x), y-x\rangle + \frac{L}{2}\|y-x\|^2.
\end{equation*}	
\end{assumption}

\begin{assumption}\label{asm-Bg}
Each loss function $f_i^{r,t}$ has bounded gradient, i.e., for any $x$, there exists a constant $B_g$ such that 
\begin{equation*}
\|\nabla f_i^{r,t}(x) \|\le B_g. 
\end{equation*}
\end{assumption}

\begin{assumption}\label{asm-regular}
For any $x^{r}$, $x^r_{\xi}$, there exists a constant $\sigma$ such that
$\|\operatorname{P}_{\mathcal{X}^{\star}}^{x^{r}}-\operatorname{P}_{\mathcal{X}^{\star}}^{x_{\xi}^{r}}\|\leq \sqrt{\sigma}\|x^r-x^r_{\xi}\|$.    
\end{assumption}

\begin{assumption}\label{asm-noncvx}
Consider the aggregated loss function $f^r$
and its optimal solution set $\mathcal{X}_r^{\star}:=\operatorname{argmin}_{x} f^r(x)$. For any $x$, there exists  constants $\alpha$ and $\mu$ such that 
\begin{equation*}
\alpha( f^r(x)-(f^r)^\star)+ \langle \nabla f^r (x), \operatorname{P}_{\mathcal{X}_r^{\star}}^{x} -x\rangle + \frac{\mu}{2}\|\operatorname{P}_{\mathcal{X}_r^{\star}}^{x} -x\|^2\leq 0.   
\end{equation*}
\end{assumption}

Assumption~\ref{asm-smooth} is standard in the optimization literature.  Assumption~\ref{asm-Bg} is frequently invoked in DP research to ensure bounded sensitivity~\cite{wei2023personalized,seif2020wireless}, and it is consistent with Lipschitz continuity of  $f_i^{r,t}$ which is often assumed in the online learning literature~\cite{denisov2022improved, cdc2021online}. Assumption~\ref{asm-regular} is a regularity condition that is necessary for our analysis since $\mathcal{X}^{\star}$ for a nonconvex problem may not be convex. 
We focus on a class of nonconvex problems that satisfy Assumption~\ref{asm-noncvx}. 
Some examples that satisfy Assumptions IV.3 and IV.4 can be found in \cite{zhou2017unified,li2018calculus}.
Below, we provide relevant examples of such problems using the following definitions.  
\begin{definition} 
For constants $\mu_{\rm QSC},\mu_{\rm WC}, c_{\rm PL}, c_{\rm EB}, c_{\rm QG}$, we introduce the following conditions of loss functions $f^r$:
\begin{itemize}
\item Quasi Strong Convexity (QSC)~\cite{necoara2019linear,zhang2013gradient}
\begin{equation*}
(f^r)^\star \ge f^r(x)+ \langle \nabla f^r (x), \operatorname{P}_{\mathcal{X}_r^{\star}}^{x} -x\rangle + \frac{\mu_{\rm QSC}}{2}\|\operatorname{P}_{\mathcal{X}_r^{\star}}^{x} -x\|^2.   
\end{equation*} 
\item Weak Convexity (WC)
\begin{equation*}
f^{r} (y)\geq f^{r} (x)+ \langle \nabla f^{r} (x), y-x\rangle - \frac{\mu_{\rm WC}}{2}\|y-x\|^2.
\end{equation*}
\item { Polyak-\L ojasiewicz Inequality (P\L)} 
\[ c_{\rm PL}\cdot(f^{r} (x)-(f^r)^\star)\leq \frac{1}{2}\|\nabla f^r(x)\|^2. \]
\item {Error Bound (EB)} 
$$c_{\rm EB}\cdot\dist(x,\mathcal{X}^\star_r)\leq	\|\nabla f^r(x)\|.$$
\item { Quadratic Growth Condition (QG)} 
$$\frac{c_{\rm QG}}{2}\cdot \dist^2(x,\mathcal{X}^\star_r)	\le f^{r} (x)-(f^r)^\star.$$
\end{itemize}
\end{definition}
If $f^r$ satisfies QSC, then Assumption~\ref{asm-noncvx} holds with $\alpha=1,\mu=\mu_{\rm QSC}$. It is well-known that, under the $L$-smoothness condition of Assumption~\ref{asm-smooth}, the $\rm P\L$, EB, and QG conditions are weaker than QSC. To illustrate this, we provide a quantitative relationship between the conditions QSC, $\rm P\L$, EB, and QG.
\begin{lemma}[Theorem~2 in \cite{karimi2016linear}]\label{lem:equi}
The aggregated loss function $f^r$ satisfies the following implications:
\[
\text{QSC} \Rightarrow \text{EB} \quad\text{and}\quad  \text{P{\L}} \Rightarrow \text{QG}
\]
with $c_{\rm EB}=\mu_{{\rm QSC}}$ and $c_{\rm QG}=c_{\rm PL}/2$.
If $f^r$ is $L$-smooth,
then $\text{EB} \Rightarrow \text{P{\L}}$ with $c_{\rm PL}=c^2_{\rm EB}/L$.
\end{lemma}

With Lemma~\ref{lem:equi},  we prove that  Assumption~\ref{asm-noncvx} holds under $\rm P\L$, EB, or QG, when the aggregated loss function $f^r$ is further assumed to be $\mu_{\rm WC}$-weakly convex with $\mu_{\rm WC}\le L$. 
\begin{corollary}\label{cor:weakcvx}
Suppose that Assumption~\ref{asm-smooth} holds and the aggregated function $f^r$ satisfies the $\mu_{\rm WC}$-weak convexity condition. Additionally, assume that one of the following conditions holds:  QG,  P{\L}, or EB with $c :=c_{\rm QG}= c_{\rm PL}/2 = c^2_{\rm EB}/2L$. If $\mu_{\rm WC} < c$, then Assumption~\ref{asm-noncvx} is satisfied for any $\alpha$ such that $\alpha \in(0, (c - \mu_{\rm WC})/L)$, with $\mu = (c - \mu_{\rm WC} - \alpha L)/2$. 
\end{corollary}
\begin{proof}
See Appendix~\ref{app-cor:weakcvx}. 
\end{proof}

In the next section, we will demonstrate how to use Assumption~\ref{asm-noncvx} to manage correlated noise and drift errors due to local updates, ultimately establishing an upper bound on the dynamic regret.

\begin{table*}[t]
\centering
\caption{Comparison of related works}
\label{tab}
\begin{threeparttable}
\begin{tabular}{llllllll}
\toprule
Work & Problem & With DP & DP Noise & Regret & Convexity & Local Updates & Regret Bound \\ 
\midrule
\multirow{2}{*}{Our} & \multirow{2}{*}{OFL} & \multirow{2}{*}{\checkmark} & Correlated & \multirow{2}{*}{Dynamic} & \multirow{2}{*}{Nonconvex} & \multirow{2}{*}{\checkmark} & \tnote{[1]}~~ $ 
\frac{1}{R\teta}  + \frac{\teta^2}{\eta_g^2} B_g^2+ (1+\frac{n}{\eta_g^2}) {\teta (\ln (R\tau))^2} \cdot\teta\frac{dB_g^2}{n\tau^2}  \frac{(\ln \frac{1}{\delta}) }{\epsilon^2} + \frac{C_{R}}{R}    
$ \\ \cline{4-4} \cline{8-8}
&  &  & Independent &  &  &  &  \tnote{[1]}~~ $ 
\frac{1}{R\teta}  + \frac{\teta^2}{\eta_g^2}  B_g^2+ (1+\frac{\teta n}{\eta_g^2})\tau \cdot {\teta}\frac{ d B_g^2}{n\tau^2}  \frac{(\ln \frac{1}{\delta}) }{ \epsilon^2} + \frac{C_{R}}{R}    
$  \\ 
\midrule
\cite{liu2023differentially} & OFL & \checkmark~\tnote{[2]} & Independent & Static & Convex & \ding{55} & 
\tnote{[2]}~~ $\frac{1}{R\eta}  + \eta B_g^2 +\eta  \frac{ d B_g^2}{n\tau^2}  \frac{(\ln \frac{1}{\delta}) }{\epsilon^2}  $    \\ 
\midrule
\cite{denisov2022improved} &  OL & \checkmark & Correlated & Static & Convex & -  &\tnote{[3]}~~ $\frac{1}{ R\eta}+\eta B_g\ln (R\tau)  \frac{\sqrt{d}B_g}{\sqrt{n}\tau}  \frac{\sqrt{\ln (\frac{1}{\delta})}}{\epsilon} +\eta B_g^2$ \\ 
\midrule
\cite{cdc2021online} & OFL & \ding{55} & - & Static & Convex & \checkmark & \tnote{[4]}~~ $\frac{1}{R\teta} + \teta B_g^2 $ \\ 
\midrule
\cite{gao2018online} &  OL & \ding{55} & - & Dynamic & Nonconvex  & - &\tnote{[5]}~~ $\frac{1}{R \eta}+ \eta + \frac{\tilde{C}_R}{R} $ \\ 
\bottomrule
\end{tabular}
\begin{tablenotes}
\small
\item [{[1]}] See \eqref{eq-27-11} and Corollary~\ref{lem-u-IID}.
\item[{[2]}] See~\cite[Theorem 1]{liu2023differentially}. For ease of comparison, we use the fully connected graph,  $(\epsilon,\delta)$-LDP, Gauss noise, and batch size as 1. 
\item[{[3]}] 
See \cite[Proposition 4]{denisov2022improved}, where we have substituted the same upper bounds of $\|b^{r,t}\|^2$ and $\|c^{r,t}\|^2$ as in our paper to enhance the original results from \cite[Proposition 4]{denisov2022improved}.
Compared our results with \cite{denisov2022improved}, the difference in dependency on $\epsilon$ and $\delta$ stems from distinct proof techniques. Simply put, we use $2a^Tb \le \|a\|^2 + \|  b \|^2$, whereas~\cite{denisov2022improved} applies $a^Tb \le \|a\|\|  b \|$.
\item[{[4]}] Here, the regret in~\cite[Theorem 1]{cdc2021online} is defined on the local models, whereas ours is defined on the released global model.
\item[{[5]}] Here, $\sum_{r=0}^{R-1} \| x_r^{\star}-x_{r+1}^{\star} \|\le \tilde{C}_R $ where $x_r^{\star}\in \operatorname{argmin}_{x}  f^r (x)$. 
\end{tablenotes}
\end{threeparttable}
\end{table*}

\subsection{Privacy-Utility  Analysis}
We begin with a lemma that quantifies the amount of noise that is needed for privacy protection. 
\begin{lemma}\label{lem-priv-analy}
Under Assumption~\ref{asm-Bg} and using the  MF mechanism \eqref{eq-abc}, if the variance of the DP noise satisfies
$$
\begin{aligned}
V_i^2 &= \frac{2B_g^2\max_{r,t} \|c^{r,t}\|^2 }{ \rho },\\
\rho &= \left(\sqrt{\epsilon + \ln \frac{1}{\delta}} - \sqrt{\ln \frac{1}{\delta}}\right)^2,    
\end{aligned}
$$ 
then Algorithm~\ref{alg-fl}  satisfies $(\epsilon,\delta)$-LDP under adaptive continual release.  Specifically, for matrix factorization technique in Example (iii), we have $V_i^2  \le  \mathcal{O} \left({\ln (R\tau)}  {B_g^2}  \frac{\ln \frac{1}{\delta} }{\epsilon^2} \right), ~\forall i.$ 
\end{lemma}

\begin{proof}
See Appendix \ref{app-lem-priv-analy}.   
\end{proof}

Next, we give a lemma that assesses the impact of DP noise on the utility. Inspired by research in the single-machine setting~\cite{koloskova2024gradient, wen2019interplay}, we use a perturbed iterate analysis technique to control the impact of the DP noise on utility. Noticing that the temporally correlated noise $(b^{r,\tau-1}-b^{r-1,\tau-1})\bxi$ in~\eqref{eq-comp-alg} represents the difference in noise between successive communication rounds, we define the virtual variable
\begin{equation*}
x^{r}_{\xi}:=x^{r}+ {\frac{\teta}{\tau}}b^{r-1,\tau-1}\bxi    
\end{equation*}
and rewrite the last step in \eqref{eq-comp-alg} as:  
\begin{equation}\label{eq-virtual}
x^{r+1}_{\xi}= x^{r}_{\xi} -\teta\cdot \frac{1}{\tau} \sum_{t=0}^{\tau-1} \frac{1}{n}\sum_{i=1}^n \nabla f_i^{r,t}(z_i^{r,t}).  
\end{equation}
Intuitively, the virtual variable $x^{r}_{\xi}$ is introduced to remove the DP noise from $x^r$. Its updates use gradient information obtained without incorporating the DP noise, as seen in \eqref{eq-virtual}. By bounding the distance between $x^{r}_{\xi}$ and the optimal solution set $\mathcal{X}^{\star}$, we establish the following lemma regarding the dynamic regret of the global models ${x^0, \ldots, x^{R-1}}$.
\begin{lemma}\label{lem-u}
Under Assumptions~\ref{asm-smooth}--\ref{asm-noncvx}, if $\teta \le \frac{{\alpha}}{8L}$, Algorithm~\ref{alg-fl} satisfies
\begin{equation*}
\frac{\reg}{R\tau}   
\le \frac{\bE\|x_{\xi}^0-\operatorname{P}_{\mathcal{X}^{\star}}^{x_{\xi}^{0}}\|^2}{{\alpha}R\teta} + \hat{S}^r,
\end{equation*} 
where 
\begin{equation}\notag
\begin{aligned}
&\hat{S}^r: =  \frac{24L(1+\sigma)\teta^2}{ \alpha^2 R} {\|\bB_R \|_F^2}  {\frac{dV^2}{\tau^2}}+
\frac{\left(12L+{\alpha}{\mu}\right)C_{R}}{{ \alpha^2R}}\\
&+ {\frac{1}{\alpha}} \left( \teta+\frac{{\alpha}}{12L} + \frac{1}{\mu}\right)2L^2\eta^2(2\tau^2 B_g^2 +{8 \max_{r,t}\|b^{r,t}\|^2 d nV^2}),
\end{aligned}	
\end{equation}
 $C_{R}:=\sum_{r=0}^{R-1}\bE\|\operatorname{P}_{\mathcal{X}_r^{\star}}^{x^{r}}-\operatorname{P}_{\mathcal{X}^{\star}}^{x^{r}}\|^2$, $V^2=V_i^2/n$, and {$\|\bB_R\|_F^2:=\|b^{0,\tau-1}\|^2+\cdots+\|b^{R-1,\tau-1}\|^2$. } 
\end{lemma}
\begin{proof}
See Appendix~\ref{app-lem-u}.     
\end{proof}

The term  $S^r$ in Lemma~\ref{lem-u} encapsulates multiple distinct sources of error,  including the $V^2$-term caused by DP noise; the $C_{R}$-term caused by the dynamic environment; and the $B_g$-term which arises from the drift error due to local updates.

Substituting Lemma~\ref{lem-priv-analy} and ${\|{\bB_R }\|_F^2 \le\mathcal{O}(R \ln (R\tau))}$ into Lemma~\ref{lem-u}, we obtain the following privacy-utility trade-off. 
\begin{theorem}[Main theorem]
\label{thm-up}
Under Assumptions~\ref{asm-smooth}-\ref{asm-noncvx}, if $\teta \le \frac{\alpha}{8L}$,  Algorithm~\ref{alg-fl} subject to $(\epsilon, \delta)$-LDP satisfies
\begin{align}\label{eq-27-11}
\frac{\reg}{R\tau} \le&\ \mathcal{O} \left( \frac{1}{R\teta}  + \frac{\teta^2}{\eta_g^2} B_g^2  \right.\\ \notag
&\left.+ \left(1+\frac{n}{\eta_g^2}\right)\teta^{2} (\ln (R\tau))^2\frac{dB_g^2(\ln \frac{1}{\delta}) }{n{\tau^2}\epsilon^2}+  \frac{C_{R}}{R} \right).    
\end{align}    
In particular, if we let $\teta=\mathcal{O}( R^{-\frac{1}{3}} (\ln (R\tau))^{-\frac{2}{3}})$, it holds that
\begin{equation*}
\begin{aligned}
\frac{\reg}{R\tau}\le\ &\mathcal{O}\left( \frac{(\ln (R\tau))^{\frac{2}{3}}}{R^{\frac{2}{3}}} +\frac{B_g^2}{\eta_g^2 R^{\frac{2}{3}}(\ln (R\tau))^{\frac{4}{3}} } \right.\\
&\left.+ \left(1+\frac{n}{\eta_g^2}\right) \frac{(\ln (R\tau))^{\frac{2}{3}} }{R^{\frac{2}{3}}}\frac{d B_g^2(\ln \frac{1}{\delta})}{n {\tau^2} \epsilon^2}+\frac{C_{R}}{R}\right).  
\end{aligned}    
\end{equation*}
\end{theorem}
\begin{proof}
See Appendix~\ref{appthm-up}.    
\end{proof}

In Theorem~\ref{thm-up}, the errors are due to the initial error, the local updates, the DP noise, and the dynamic environment, respectively. We have the following observations. 
\begin{itemize}


\item Our perturbed iterate analysis effectively controls the impact of DP noise on the utility. The DP noise error term is $\mathcal{O}((\ln(R\tau))^{\frac{2}{3}}/R^{\frac{2}{3}})$, which decreases as $R$ increases. The theoretical advantages of correlated noise over independent noise are further discussed in Section~\ref{sec-IID-nonIID}.  

\item The term $C_{R}:=\sum_{r=0}^{R-1}\bE\|\operatorname{P}_{\mathcal{X}_r^{\star}}^{x^{r}}-\operatorname{P}_{\mathcal{X}^{\star}}^{x^{r}}\|^2$ captures the changes in the solution set $\mathcal{X}_r^{\star}$ over time relative to a fixed solution set $\mathcal{X}^{\star}$. This variation  
is unavoidable in dynamic regret~\cite[Theorem 5]{zhang2017improved}. Intuitively, when the environment changes rapidly, online learning algorithms face greater challenges in achieving high utility. 

\item Establishing a sublinear regret bound for nonconvex problems, even for static regret, poses significant challenges~\cite[Proposition 3]{suggala2020online}. In section~\ref{sec-static}, we show that the error due to $C_R$ can be improved when we consider static regret bound under the strongly convex (SC) condition. 
\end{itemize}

\subsection{Comparison of Correlated Noise and Independent Noise}\label{sec-IID-nonIID}

If we replace the correlated noise in Algorithm~\ref{alg-fl} with independent noise for privacy protection, as specified in  \eqref{eqn:indp}, we can derive the following results.
\begin{corollary}[Independent noise]\label{lem-u-IID}
Under Assumptions~\ref{asm-smooth}-\ref{asm-noncvx}, if $\teta \le \frac{{\alpha}}{8L}$, Algorithm~\ref{alg-fl} with independent noise subject to $(\epsilon,\delta)$-LDP satisfies
\begin{equation*}
\begin{aligned}
\frac{\reg}{R\tau} 
\le\mathcal{O}\! \left(\! \frac{1}{R\teta}  + \frac{\teta^2B_g^2}{\eta_g^2}  + \Big(1+\frac{\teta n}{\eta_g^2}\Big)   \frac{{\teta} d B_g^2}{n{\tau}}  \frac{(\ln \frac{1}{\delta}) }{\epsilon^2} + \frac{C_{R}}{R}\!\right).     
\end{aligned}        
\end{equation*}   
\end{corollary}
\begin{proof}
See Supplementary.     
\end{proof}
Comparing Corollary~\ref{lem-u-IID} for independent noise with \eqref{eq-27-11} for correlated noise, we find that the use of correlated noise results in the smaller regret bound when {$\teta\le   \frac{\tau}{(1+{n}/{\eta_g^2}) (\ln(R\tau))^2}$.}

\subsection{Static Regret Under Strongly Convex (SC) Condition}\label{sec-static}
To improve the dependence on $C_R$, we establish a static regret bound under the SC condition.

\begin{corollary}[Static regret under SC]\label{coro}
Assume the loss function $f^r$ to be strongly convex, i.e., there exists a constant $\mu_{\rm SC}$ such that 
\begin{equation*}
\begin{aligned}
f^r(x)-f^r(x^\star)+\frac{\mu_{\rm SC}}{2}\|x-x^\star\|^2
\le \left\langle  \nabla f^r(x), x-x^\star\right\rangle, ~\forall x,
\end{aligned}      
\end{equation*}
where $x^{\star}:=\operatorname{argmin}_x \sum_{r=0}^{R-1} f^r(x)$.
Then under Assumpitions~\ref{asm-smooth}--\ref{asm-regular}, if $\teta\le \frac{1}{10L(\ln(R\tau))}$ and $\teta=\mathcal{O}((R\ln(R\tau))^{-\frac{1}{2}})$, Algorithm~\ref{alg-fl} subject to $(\epsilon, \delta)$-LDP satisfies
\begin{equation*}
\begin{aligned}
& \frac{\regs}{R\tau} 
\le \mathcal{O}\left(\frac{(\ln(R\tau))^{\frac{1}{2}}}{R^{\frac{1}{2}}} +\frac{1}{R\ln(R\tau)}\frac{B_g^2}{\eta_g^2 } \right.\\
&\left.+\left(1+\frac{n}{\eta_g^2}\right)\frac{(\ln(R\tau))^{\frac{1}{2}}}{R^{\frac{1}{2}}}\frac{d B_g^2(\ln \frac{1}{\delta})}{n {\tau^2}\epsilon^2} + \frac{(\ln(R\tau))^{\frac{1}{2}}}{R^{\frac{1}{2}}} \frac{C_{R}}{R}
\right),
\end{aligned}    
\end{equation*}  
where $C_R:=\sum_{r=0}^{R-1} \|x^{\star}-x_r^{\star}\|^2 $ and $x_r^{\star} :=\operatorname{argmin}_x f^r(x)$. 
\end{corollary}
\begin{proof}
See Supplementary.     
\end{proof}

In Corollary~\ref{coro}, the error term caused by $C_{R}$ converges at the rate of $\mathcal{O}((\ln(R\tau))^{\frac{1}{2}}/R^{\frac{3}{2}}) $ for a static regret under SC.

\begin{remark}
Our analysis of nonconvex OFL is novel, even without relying on LDP. In contrast to prior work on general nonconvex \cite{suggala2020online} and pseudo-convex \cite{gao2018online} online settings—both of which assume learners have access to offline optimization oracles and achieve $\mathcal{O}(R^{\frac{1}{2}})$ regret bounds—our approach establishes a tighter $\mathcal{O}(R^{\frac{1}{3}})$ bound, as shown in Theorem \ref{thm-up}. This improvement is particularly significant as it demonstrates that better regret guarantees are achievable without requiring offline oracles.
While \cite{zhang2017improved} achieves improved bounds under a semi-strong convexity condition, their analysis fundamentally depends on convexity. Our algorithm differs from these prior methods in several other ways; see Section I.B.3 for further details.
Overall, this work takes a step toward developing nonconvex OFL methods tailored to loss functions with a particular structure, leading to improved regret guarantees. To the best of our knowledge, these results and techniques are novel and cannot be directly derived from existing research on nonconvex online optimization.
\end{remark}

\section{Numerical Experiments}
{We implement our algorithm with the three MF mechanisms discussed above: MF (i) is the binary tree~\cite{guha2013nearly}, MF (ii) is the optimized factorization~\cite{denisov2022improved} and MF (iii) is the Toeplitz matrix construction~\cite{jaja2023almost}. These variations are then compared to the algorithms in \cite{cdc2021online} (which does not add privacy-preserving noise) and \cite{liu2023differentially} (which adds independent noise).
For a fair comparison, we modify the mini-batch SGD of size \(\tau\) in \cite{liu2023differentially} to \(\tau\) local updates, consistent with our approach.}
Each experiment is conducted 10 times, with the results averaged and displayed alongside error bars representing the standard deviation.

\subsection{Logistic Regression}

We consider the following logistic regression problem:
\[
\min_{x\in\mathbb{R}^{1\times d}} \; \frac{1}{R \tau n}\sum_{r=0}^{R-1}\sum_{t=0}^{\tau-1}\sum_{i=1}^n f_i^{r,t}(x),
\]
where the loss function for learner \(i\) is \[
f_i^{r,t}(x) = \log \left(1 + \exp \left( -\left( x \mathbf{a}_{i}^{r,t} \right) b_{i}^{r,t} \right) \right),
\]
with \( (\mathbf{a}_{i}^{r,t}, b_{i}^{r,t}) \in \mathbb{R}^{d} \times \{-1,+1\} \) representing the feature-label pairs. The data is generated using the method described in \cite{li2020federated}, which allows us to control the degree of heterogeneity using two parameters, \(\alpha\) and \(\beta\).

In our first set of experiments, we set the dimensionality to \(d = 100\) and use \(n = 20\) learners. Each learner is responsible for 4000 clients, who arrive sequentially in steps, with \(\tau = 4\) and \(R = 1000\). The heterogeneity parameters are set to \((\alpha, \beta) = (0.1, 0.1)\), and we experiment with two different privacy budgets: \((\epsilon, \delta) \in \{(2, 10^{-3}), (0.5, 10^{-3})\}\).

The results are presented in Fig.~\ref{fig-lr}. Of the four curves shown, all except the one corresponding to independent noise (which uses a smaller step size) share the same step size. As seen in Fig.~\ref{fig-lr}, under both privacy budgets, the curves for our algorithms with binary tree, optimized factorization, and Toeptitz matrix closely follow the curve of the noiseless case. With a stricter LDP budget \((0.5, 10^{-3})\), the variance in our algorithms increases slightly. {Our algorithms with optimized factorization and Toeplitz matrix outperform binary tree, consistent with the findings in \cite{denisov2022improved, jaja2023almost}.} In contrast, under privacy budgets \((2, 10^{-3})\) and \((0.5, 10^{-3})\),  the method with independent noise has to use a small step size, leading to low accuracy. This highlights the clear advantage of using correlated noise over independent noise.

\begin{figure}
\begin{subfigure}{0.5\textwidth}
  \centering
  \includegraphics[width=7.5cm]{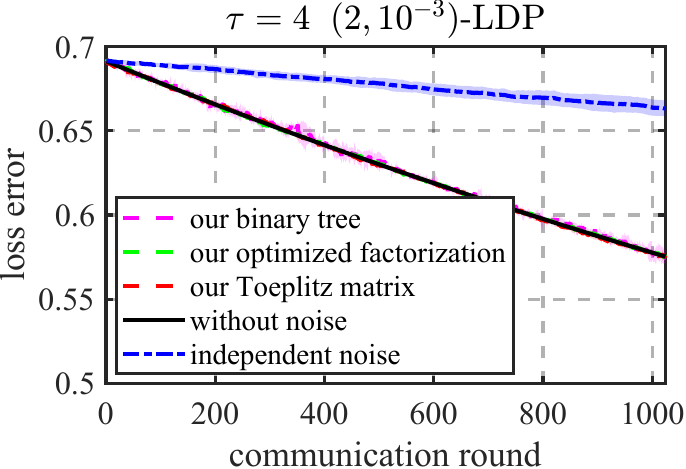}
\end{subfigure}
\vspace{1mm}
\\ 
\begin{subfigure}{0.5\textwidth}
  \centering
  \includegraphics[width=7.5cm]{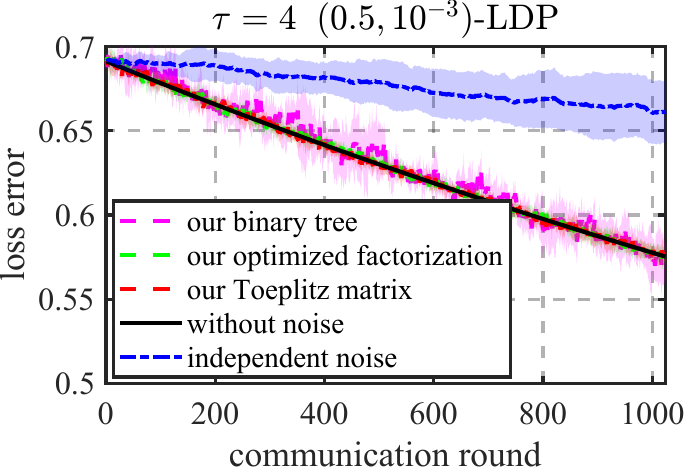}
\end{subfigure}
\caption{Comparison on logistic regression}
\label{fig-lr}
\end{figure}

\begin{figure*}[htbp]
\begin{minipage}[t]{.246\linewidth}
\centering
\centerline{\includegraphics[width=4.55cm]{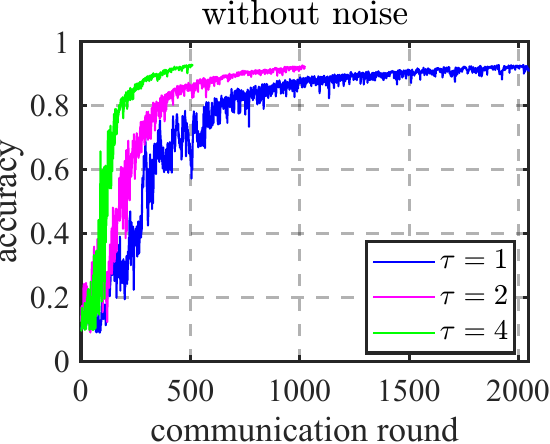}}
\end{minipage}
\begin{minipage}[t]{0.246\linewidth}
\centering
\centerline{\includegraphics[width=4.55cm]{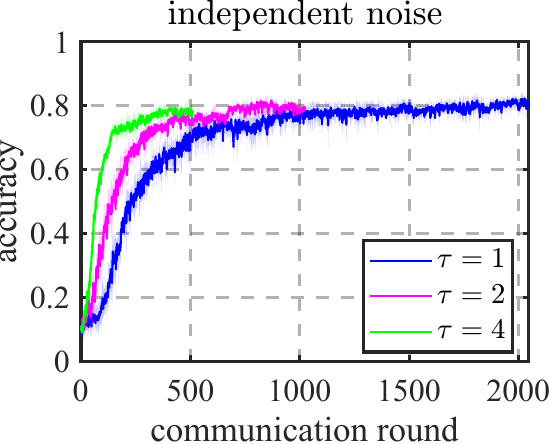}}
\end{minipage}
\begin{minipage}[t]{0.246\linewidth}
\centering
\centerline{\includegraphics[width=4.55cm]{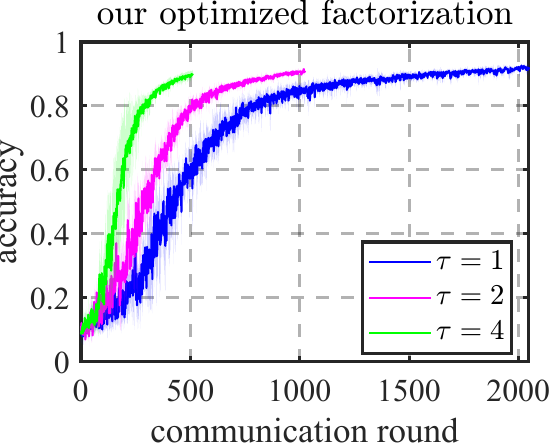}}
\end{minipage}
\begin{minipage}[t]{0.246\linewidth}
\centering
\centerline{\includegraphics[width=4.55cm]{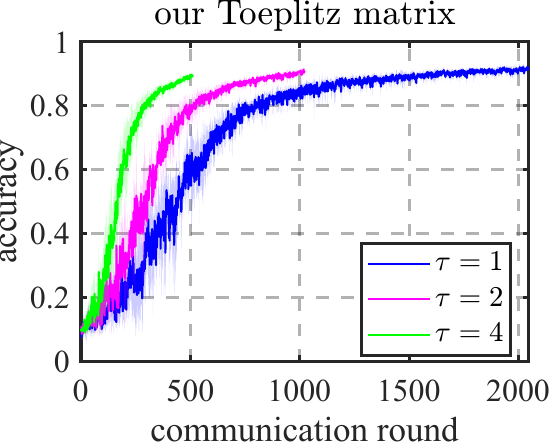}}
\end{minipage}
\caption{{Ablation and comparison on CNN classification under $(2,10^{-3})$-LDP budget}}
\label{fig-cnn}
\end{figure*}

\subsection{Training of Convolutional Neural Networks}

We explore the training of a convolutional neural network (CNN) using the MNIST dataset~\cite{zhang2024secure}.
The CNN architecture includes two convolutional layers, each with 32 filters of size \(3 \times 3\) and a max-pooling layer of size \(2 \times 2\). These layers feed to two fully connected layers, containing 64 and 10 units, respectively. The hidden layers employ ReLU activation functions, while the output layer uses a softmax activation. The training is performed using the cross-entropy loss function.

The MNIST dataset, containing 60,000 images of hand-written digits (0-9), is used for training. To introduce data heterogeneity, 30,000 images are randomly distributed evenly across 10 learners (3,000 per learner). The remaining 30,000 images are distributed unevenly, with all samples of digit \(l\) assigned to learner \(l+1\).  We use 10,000 samples for testing to evaluate the accuracy of the global model on the server. The privacy budget is set to \((2, 10^{-3})\)-LDP, and \(\tau \in \{1,2,4\}\). 

In this set of experiments, we compare algorithms with correlated noise, independent noise, and without DP noise. We set different numbers of local updates \( \tau \) to show its impact on communication efficiency while fixing \( R\tau \) for a fair comparison, ensuring that the total number of data points for all runs remains the same. Our algorithm leverages optimized factorization and Toeplitz matrices to construct correlated noise. For all the algorithms, we use the same step sizes $\eta$ and $\eta_g$. The results are shown in Fig.~\ref{fig-cnn}.  

First, compared to independent noise, using correlated noise results in higher final accuracy. With independent noise, the final accuracy is approximately 0.8, whereas our algorithms, utilizing optimized factorization and Toeplitz matrices, achieve around 0.9.  This highlights the benefits of correlated noise over independent noise.  
The final accuracy of both the noise-free and correlated-noise algorithms is similar. However, in the early stages, the accuracy with correlated noise is lower than in the noise-free case. For example, when \( \tau = 1 \) and the communication round is 500, the accuracy of the noise-free algorithm is around 0.7, while our correlated-noise algorithms reach approximately 0.6.  

Furthermore, we examine the impact of different values of \( \tau \). All four subplots indicate that increasing \( \tau \) from 1 to 4 reduces the number of communication rounds while maintaining a similar level of accuracy. This is because a larger \( \tau \) decreases the communication frequency between learners and the server, thereby reducing the total number of communication rounds.  
Note, however, that  \( \tau \) should not be taken too large. As shown in \eqref{eq-27-11}, when \( \eta \eta_g \) is fixed and \( \tau \) varies, the second error term caused by local updates satisfies \( \frac{\teta^2 B_g^2}{\eta_g^2} = \tau^2 \eta^2 B_g^2 \). Hence, if \( \tau \) is increased too much, the drift error will exceed the error introduced by the DP noise and the accuracy will be reduced.

\section{Conclusions}
We have proposed an LDP algorithm for OFL that uses temporally correlated noise to protect client privacy under adaptive continual release.  To address the challenges caused by DP noise and local updates with streaming non-IID data, we used a perturbed iterate analysis to control the impact of the DP noise on the utility. Moreover, we demonstrated how the drift error from local updates can be managed under a class of  nonconvex loss functions. Subject to a fixed DP budget, we established a dynamic regret bound that explicitly shows the trade-off between utility and privacy.  Numerical results demonstrated the efficiency of our algorithm.

\bibliographystyle{IEEEtran}
\bibliography{autosam} 

\appendix

\subsection{Proof of Corollary \ref{cor:weakcvx}}\label{app-cor:weakcvx}
On the one hand, by weak convexity of $f^r$, we have
\begin{equation*}
f^r(y) \geq f^r(x)+\langle\nabla f^r(x), y-x\rangle-\frac{\mu_{\rm WC}}{2}\|y-x\|_2^2
\end{equation*}
for any $y\in \mathbb{R}^d$. 
Moreover, Lemma \ref{lem:equi} has established that P{\L} (resp. EB) with constant $2c$ (resp. $\sqrt{2cL}$) implies QG with the constant $c$, \emph{i.e.} that 
\begin{align*}
    f^r(x) - (f^r)^{\star} &\geq \frac{c}{2}\Vert x-y\Vert_2^2
\end{align*}
for $y=P_{{\mathcal X}_r^{\star}}^x$. Combining the inequalities for this $y$ yields
\begin{align}\label{eq:wcvxkey1}
0 &\geq \langle\nabla f^r(x), \operatorname{P}_{\mathcal{X}_r^{\star}}^x-x\rangle  + \frac{c-\mu_{\rm WC}}{2}\|x-\operatorname{P}_{\mathcal{X}_r^{\star}}^x\|_2^2.
\end{align}

On the other hand, we know from Assumption~\ref{asm-smooth} that
\begin{align}\label{eq:wcvxkey2}
0
&\ge\frac{1}{n\tau}\sum_{i=1}^n\sum_{t=0}^{\tau-1} (f_i^{r,t} (x)-f_i^{r,t} (\operatorname{P}_{\mathcal{X}_r^{\star}}^x))-\frac{L}{2}\|x-\operatorname{P}_{\mathcal{X}_r^{\star}}^x\|_2^2\notag\\
&= f^r(x)-(f^r)^\star-\frac{L}{2}\|x-\operatorname{P}_{\mathcal{X}_r^{\star}}^x\|_2^2,
\end{align}
where we used $\nabla f^r(\operatorname{P}_{\mathcal{X}_r^{\star}}^x)=0$ in the first inequality.

Mutiplying \eqref{eq:wcvxkey2} by $\alpha\in (0, (c-\mu_{\rm WC})/L)$ and adding the resulting inequality to \eqref{eq:wcvxkey1} yields
\begin{equation}
\begin{aligned}
0 \ge &\ \alpha( f^r(x)-(f^r)^\star)+ \langle \nabla f^r (x), \operatorname{P}_{\mathcal{X}_r^{\star}}^{x} -x\rangle \\
&+ \frac{c-\mu_{\rm WC}-\alpha L}{2}\|\operatorname{P}_{\mathcal{X}_r^{\star}}^{x} -x\|^2.     
\end{aligned}    
\end{equation}
Thus, we have proved that Assumption~\ref{asm-noncvx} is satisfied with $\mu = (c - \mu_{\rm WC} - \alpha L)/2>0$ for any $\alpha$ such that $\alpha \in(0, (c - \mu_{\rm WC})/L)$, which completes the proof of Corollary \ref{cor:weakcvx}.


\subsection{Proof of Lemma~\ref{lem-priv-analy}}\label{app-lem-priv-analy} 
To simplify the derivation, we base our privacy analysis on the concept of $\rho$-zCDP \cite{bun2016concentrated}, where $\rho>0$ is a parameter to measure the privacy loss. According to \cite[Proposition 1.3]{bun2016concentrated}, $\rho$-zCDP can be transferred to $(\epsilon,\delta)$-LDP via 
\begin{equation}
\rho = \left(\sqrt{\epsilon + \ln \frac{1}{\delta}} - \sqrt{\ln \frac{1}{\delta}}\right)^2
\approx \frac{\epsilon^2}{4 \ln \frac{1}{\delta}}.
\end{equation}
In the following, we will establish the relationship between the parameters $\rho$ and $V_i^2$. 

By using the MF technique, our algorithm adds Gaussian DP noise $\bxi_i$ into $\bC \bG_i$ via the Gaussian mechanism 
\begin{equation*}
\bC \bG_i + \bxi_i.     
\end{equation*} 
As shown in \cite[Theorem 2.1]{denisov2022improved}, the MF technique can protect privacy under adaptive continual release, and the parameters are the same as in the non-adaptive continual release setting. 
It is therefore enough for us to analyze privacy in a non-adaptive setting. With 
$\bG_i=\{\nabla f_i^{0,0}(z_i^{0,0});\ldots;\nabla f_i^{R-1,\tau-1}(z_i^{R-1,\tau-1})\}$ and $\bG_i'$ denoting the corresponding gradients evaluated on a neighboring datasets $D_i'$, the sensitivity $\Delta$ of $\bC\bG_i$ satisfies  
\begin{equation}\label{eq-sens}
\begin{aligned}
\Delta:=&\ \| \bC (\bG_i-\bG'_i) \|_F \\
=&\ {\| c^{0,0} \left(\nabla f_i^{0,0}(z_i^{0,0})-\nabla f_i^{0,0'}(z_i^{0,0})\right) \|_F}\\
\le&\ {\max_{r,t} \|c^{r,t}\| \left({2B_g}\right)}.   
\end{aligned}
\end{equation}
Here, without loss of generality, we assume that $D_i$ and $D_i'$ differ in the first entry, implying that $\nabla f_i^{0,0}\neq \nabla f_i^{0,0'}$. The last inequality follows from Assumption~\ref{asm-Bg}.

By~\cite[Proposition 1.6]{bun2016concentrated}, the Gaussian mechanism with variance $V_i^2$ satisfies \(\rho\)-zCDP with \( \rho = \frac{\Delta^2}{2V_i^2} \). Combining this fact with ~\eqref{eq-sens}, we find that our algorithm is $\rho$-zCDP if
\begin{equation}\label{eq-vi}
\begin{aligned}
V_i^2 = \frac{\Delta^2}{2\rho} \le \max_{r,t} \|c^{r,t}\|^2 \left({2B_g}\right)^2\frac{1}{2\rho}, ~\forall i.
\end{aligned}    
\end{equation}
%
This result holds for all MF techniques (i)--(iii). Specifically, for MF technique (iii), by substituting the upper bound of $\max_{r,t} \|c^{r,t}\|^2$ given in \eqref{eq-exmp-iii} into \eqref{eq-vi}, we have 
\begin{equation}\notag
\begin{aligned}
V_i^2  \le  \mathcal{O} \left({\ln (R\tau)}  {B_g^2}  \frac{\ln \frac{1}{\delta} }{\epsilon^2} \right), ~\forall i,
\end{aligned}    
\end{equation} 
which completes the proof of Lemma~\ref{lem-priv-analy}.

\subsection{Proof of Lemma~\ref{lem-u}}\label{app-lem-u}

In the following analysis, we consider MF technique (iii) with established theoretical bounds for $\|b^{r,t}\|^2$ and $\|c^{r,t}\|^2$. 

As shown in \eqref{eq-comp-alg}, mathematically, the local and global updates of our Algorithm~\ref{alg-fl} can be rewritten as
\begin{equation}
\begin{aligned}
z_i^{r,t+1} &=z_i^{r,t}-\eta \left(\nabla f_i^{r,t}(z_i^{r,t})+{(b^{r,t}-b^{r,t-1})\bxi_i}\right), \\
x^{r+1}&=x^r-\teta \Big( g^r + {\frac{1}{\tau}(b^{r, \tau-1}-b^{r-1,\tau-1})\bxi}\Big),
\end{aligned}
\end{equation}
where $z_i^{r,0}=x^r$, $\teta=\eta_g\eta\tau$, 
\begin{equation*}
g^r =\frac{1}{n\tau} \sum_{t=0}^{\tau-1}\sum_{i=1}^n \nabla f_i^{r,t}(z_i^{r,t}),   
\end{equation*}
and $\bxi=\frac{1}{n}\sum_{i=1}^n \bxi_i \sim \mathcal{N}(0,V^2)^{R\tau\times d}$ with $V^2=V_i^2/n$. 

We aim to achieve tighter convergence by exploiting the structure of temporally correlated noise. To this end, we define 
\begin{equation*}
x^{r}_{\xi}:=x^{r}+ {\frac{\teta}{\tau}} b^{r-1,\tau-1}\bxi    
\end{equation*}
and re-write the global update in terms of this new variable
\begin{equation}\label{eq-x-xi}
x^{r+1}_{\xi}= x^{r}_{\xi} -\teta g^r.   
\end{equation}
The corresponding iterates satisfy
\vspace{-1mm}
\begin{align}
&\bE\|x_{\xi}^{r+1}-\operatorname{P}_{\mathcal{X}^{\star}}^{x_{\xi}^{r+1}}\|^2 \label{eq-cvx-r+1-r}\\ \notag
\leq\ &\bE \|x_{\xi}^r-\operatorname{P}_{\mathcal{X}^{\star}}^{x_{\xi}^{r}}-\teta g^r \|^2 \\ \notag
=\ & \bE\|x_{\xi}^r-\operatorname{P}_{\mathcal{X}^{\star}}^{x_{\xi}^{r}}\|^2 + {\teta^2\bE\|g^r\|^2}\\\notag
&{-2\teta \bE\left\langle g^r,x_{\xi}^r-x^r+x^r-\operatorname{P}_{\mathcal{X}^{\star}}^{x^{r}}+\operatorname{P}_{\mathcal{X}^{\star}}^{x^{r}}-\operatorname{P}_{\mathcal{X}^{\star}}^{x_{\xi}^{r}}\right\rangle}\\\notag
\le\ & \bE\|x_{\xi}^r-\operatorname{P}_{\mathcal{X}^{\star}}^{x_{\xi}^{r}} \|^2 \underbrace{-2\teta \bE\left\langle g^r,x^r-\operatorname{P}_{\mathcal{X}^{\star}}^{x^{r}}\right\rangle}_{\rm (IV)}+\teta^2\left(1\!+\!\frac{\theta}{\teta}\right)\underbrace{\bE\|g^r\|^2}_{\rm (V)}\\ \notag
&+ \frac{2\teta}{\theta}\left(\bE\|x_{\xi}^r-x^r\|^2+\bE\|\operatorname{P}_{\mathcal{X}^{\star}}^{x^{r}}-\operatorname{P}_{\mathcal{X}^{\star}}^{x_{\xi}^{r}}\|^2\right),\notag
\end{align} 
where the first inequality follows from the optimality of $\operatorname{P}_{\mathcal{X}^{\star}}^{x_{\xi}^{r+1}}$ and the second inequality uses 
{Young's inequality }$2a^Tb\le \theta_1\|a\|^2+\frac{1}{\theta_1} \|  b \|^2$ for $\theta_1>0$ { with  $\theta_1=\theta $.}
To bound (IV), we use the definition of $g^r$ and add and subtract $\nabla f_i^{r,t}(x^r)$ to find
\begin{align}\label{eq-IV-decomp}
\vspace{-1mm}
&{\rm (IV)}=-2\teta \bE\Big\langle \frac{1}{n\tau}\sum_{t=0}^{\tau-1}\sum_{i=1}^n \nabla f_{i}^{r,t}(x^{r}),x^r-\operatorname{P}_{\mathcal{X}^{\star}}^{x^{r}}\Big\rangle\\
&-2\teta \bE\Big\langle \frac{1}{n\tau}\sum_{t=0}^{\tau-1}\sum_{i=1}^n (\nabla f_{i}^{r,t}(z_{i}^{r,t})-\nabla f_{i}^{r,t}(x^{r})),x^r-\operatorname{P}_{\mathcal{X}^{\star}}^{x^{r}}\Big\rangle.\notag
\end{align}    
We refer to the two terms on the right hand of the above equality as (IV.I) and (IV.II), respectively.  For (IV.I), we use Assumption \ref{asm-noncvx} to get 
\begin{equation*}
\begin{aligned}
&{\rm (IV.I)}\\
= &  -2\teta \bE\Big\langle \frac{1}{n\tau}\sum_{t=0}^{\tau-1}\sum_{i=1}^n \nabla f_{i}^{r,t}(x^{r}),x^r-\operatorname{P}_{\mathcal{X}_r^{\star}}^{x^{r}}+\operatorname{P}_{\mathcal{X}_r^{\star}}^{x^{r}}-\operatorname{P}_{\mathcal{X}^{\star}}^{x^{r}}\Big\rangle\\
\leq  &-\frac{2\teta{\alpha}}{n\tau}\sum_{i=1}^n\sum_{t=0}^{\tau-1} \bE(f_i^{r,t}(x^r)-(f^{r})^\star)-{\teta\mu}\bE\|x^r-\operatorname{P}_{\mathcal{X}_r^{\star}}^{x^{r}}\|^2\\
& +\teta \left(\theta\bE \Big\| \frac{1}{n\tau} \sum_{i=1}^n\sum_{t=0}^{\tau-1}\nabla f_i^{r,t}(x^r)
\Big\|^2+\frac{1}{\theta}\bE\|\operatorname{P}_{\mathcal{X}_r^{\star}}^{x^{r}}-\operatorname{P}_{\mathcal{X}^{\star}}^{x^{r}}\|^2\right),
\end{aligned}    
\end{equation*}
where the inequality follows from $2a^Tb\le \theta\|a\|^2+\frac{1}{\theta} \|  b \|^2$ for $\theta>0$ and the observation that Assumption~\ref{asm-noncvx} implies that
\begin{equation*}
\begin{aligned}
& \Big\langle \frac{1}{n\tau}\sum_{i=1}^n\sum_{t=0}^{\tau-1} \nabla f_i^{r,t}(x), x-\operatorname{P}_{\mathcal{X}_r^{\star}}^{x}\Big\rangle\\
 \ge &\ \frac{{\alpha}}{n\tau}\sum_{i=1}^n\sum_{t=0}^{\tau-1} (f_i^{r,t}(x)-(f^{r})^\star)+\frac{\mu}{2}\|x-\operatorname{P}_{\mathcal{X}_r^{\star}}^{x}\|^2, ~\forall x.
\end{aligned}    
\end{equation*}
For (IV.II), we have
\vspace{-1mm}
\begin{equation*}
\begin{aligned}
&{\rm (IV.II)}\\
=&-\!2\teta \bE\Big\langle \!\frac{1}{n\tau}\sum_{t=0}^{\tau-1}\sum_{i=1}^n (\nabla f_{i}^{r,t}(z_{i}^{r,t})\!-\!\nabla f_{i}^{r,t}(x^{r})),x^r-\operatorname{P}_{\mathcal{X}^{\star}}^{x^{r}}\!\Big\rangle \\
\le\ &  \frac{2\teta}{\mu} \bE\left\|\frac{1}{n\tau}\sum_{t=0}^{\tau-1}\sum_{i=1}^n (\nabla f_{i}^{r,t}(z_{i}^{r,t})-\nabla f_{i}^{r,t}(x^{r})) \right\|^2\notag\\
&+ \frac{\teta \mu}{2}\bE\| x^r-\operatorname{P}_{\mathcal{X}^{\star}}^{x^{r}}\|^2\\
\le\ &  \frac{2\teta}{\mu} \bE\left\|\frac{1}{n\tau}\sum_{t=0}^{\tau-1}\sum_{i=1}^n (\nabla f_{i}^{r,t}(z_{i}^{r,t})-\nabla f_{i}^{r,t}(x^{r})) \right\|^2\notag\\
&+ \teta \mu\bE\left(\| x^r-\operatorname{P}_{\mathcal{X}_r^{\star}}^{x^{r}}\|^2+\|\operatorname{P}_{\mathcal{X}_r^{\star}}^{x^{r}}-\operatorname{P}_{\mathcal{X}^{\star}}^{x^{r}}\|^2\right),
\end{aligned}    
\end{equation*}
where we use Young's inequality with $\theta_1=2/\mu$ in the first inequality and triangle inequality in the second inequality. 
Now, substituting (IV.I) and (IV.II) into (IV) yields
\begin{equation*}
\begin{aligned}
{\rm (IV)}\le & 
-\frac{2\teta{\alpha}}{n\tau}\sum_{i=1}^n\sum_{t=0}^{\tau-1} \bE(f_i^{r,t}(x^r)-(f^{r})^\star)\\
&+\teta \theta\bE \left\| \frac{1}{n\tau} \sum_{i=1}^n\sum_{t=0}^{\tau-1}\nabla f_i^{r,t}(x^r)
\right\|^2\\
&+\teta\left(\frac{1}{\theta}+\mu\right)\bE\|\operatorname{P}_{\mathcal{X}_r^{\star}}^{x^{r}}-\operatorname{P}_{\mathcal{X}^{\star}}^{x^{r}}\|^2\\
& +\frac{2\teta}{\mu}\bE \left\|\frac{1}{n\tau}\sum_{t=0}^{\tau-1}\sum_{i=1}^n (\nabla f_{i}^{r,t}(z_{i}^{r,t})-\nabla f_{i}^{r,t}(x^{r})) \right\|^2.
\end{aligned}    
\end{equation*}
Next, for ({\rm V}) in~\eqref{eq-cvx-r+1-r}, we add and subtract $\nabla f_i^{r,t}(x^r)$ to find
\begin{equation*}
\begin{aligned}
&{\rm (V)}\\  
=\ & \bE\left\| \frac{1}{n\tau} \sum_{i=1}^n\sum_{t=0}^{\tau-1} (\nabla f_i^{r,t}(z_i^{r,t})-\nabla f_i^{r,t}(x^r)+\nabla f_i^{r,t}(x^r)) \right\|^2\\
\le\ & 2\bE\left\| \frac{1}{n\tau} \sum_{i=1}^n\sum_{t=0}^{\tau-1} (\nabla f_i^{r,t}(z_i^{r,t})-\nabla f_i^{r,t}(x^r))\right\|^2\\
&+2\bE \left\| \frac{1}{n\tau} \sum_{i=1}^n\sum_{t=0}^{\tau-1}\nabla f_i^{r,t}(x^r)\right\|^2.
\end{aligned}    
\end{equation*}
Substituting (IV) and (V) into~\eqref{eq-cvx-r+1-r}, we have 
\begin{equation}\label{eq-16}
\begin{aligned}
&\bE\|x_{\xi}^{r+1}-\operatorname{P}_{\mathcal{X}^{\star}}^{x_{\xi}^{r+1}}\|^2 -\bE\|x_{\xi}^{r}-\operatorname{P}_{\mathcal{X}^{\star}}^{x_{\xi}^{r}}\|^2 \\
&\le-\frac{2\teta{\alpha}}{n\tau}\sum_{i=1}^n\sum_{t=0}^{\tau-1} \bE(f_i^{r,t}(x^r)-(f^{r})^\star)\\
&+2c_{{\rm VI}}\cdot\underbrace{\bE\Big\| \frac{1}{n\tau} \sum_{i=1}^n\sum_{t=0}^{\tau-1} (\nabla f_i^{r,t}(z_i^{r,t})-\nabla f_i^{r,t}(x^r))\Big\|^2}_{\rm (VI)}\\
&+2c_{{\rm VII}}\cdot \underbrace{\bE \left\| \frac{1}{n\tau} \sum_{i=1}^n\sum_{t=0}^{\tau-1}\nabla f_i^{r,t}(x^r)
\right\|^2}_{\rm (VII)}\\
&+ \frac{2\teta}{\theta}\left(\bE\|x_{\xi}^r-x^r\|^2+\bE\|\operatorname{P}_{\mathcal{X}^{\star}}^{x^{r}}-\operatorname{P}_{\mathcal{X}^{\star}}^{x_{\xi}^{r}}\|^2\right)\\
&+\teta\left(\frac{1}{\theta}+\mu\right)\bE\|\operatorname{P}_{\mathcal{X}_r^{\star}}^{x^{r}}-\operatorname{P}_{\mathcal{X}^{\star}}^{x^{r}}\|^2,
\end{aligned}       
\end{equation}
where $c_{{\rm VI}}:=\teta^2\left(1+\frac{\theta}{\teta}\right) + \frac{\teta}{\mu}$ and $c_{{\rm VII}}:=\teta^2\left(1+\frac{3\theta}{2\teta}\right)$.

We use $L$-smoothness to bound the term (VI)
\begin{equation*}
{\rm (VI)}
\leq L^2\frac{1}{n\tau} \sum_{i=1}^n\sum_{t=0}^{\tau-1} \bE\|z_i^{r,t}-x^r\|^2. 
\end{equation*}
By repeatedly applying the local updates {and substituting $b^{r,-1}=b^{r-1,\tau-1}$}, we thus have
$z_i^{r,t}=z_{i}^{r,0}-\eta  \Big(\nabla f_i^{r,0}(z_i^{r,0})+\cdots+\nabla f_i^{r,t-1}(z_i^{r,t-1}) + (b^{r,t-1}-b^{r-1,\tau-1})\bxi_i\Big)$ for all $t$.  Since $z_i^{r,0}=x^r$, it follows that  
\begin{align}\label{eq-18}
&\bE\|z_i^{r,t}-x^r\|^2\\ \notag
=&\ \eta^2 \bE\Big\|  \sum_{\tilde{t}=0}^{t-1} \nabla f_i^{r,\tilde{t}}(z_i^{r,\tilde{t}}) + {(b^{r,t-1}-b^{r-1,\tau-1})\bxi_i} \Big\|^2 \\\notag
\le&\  \eta^2\bE\left(2\Big\|  \sum_{\tilde{t}=0}^{t-1} \nabla f_i^{r,\tilde{t}}(z_i^{r,\tilde{t}})\Big\|^2 + 2\Big\|{(b^{r,t-1}-b^{r-1,\tau-1})\bxi_i} \Big\|^2\right) \\\notag
\le&\ \eta^2 (2\tau^2 B_g^2 +{8 \max_{r,t}\|b^{r,t}\|^2 d V_i^2}),   
\end{align}    
{where we use
$2\bE\|{(b^{r,t-1}-b^{r-1,\tau-1})\bxi_i}\|^2\le 4 (\bE\|b^{r,t-1}\bxi_i\|^2$  $+\bE\|b^{r-1,\tau-1}\bxi_i\|^2 )$
and the bound $\bE \|b^{r,t}\bxi_i\|^2= \|b^{r,t}\|^2 dV_i^2 \le$ $\max_{r,t} \|b^{r,t}\|^2d V_i^2$ for all $r,t$ in the last step.}

Substituting~\eqref{eq-18} into (VI) {and $V_i^2=n V^2$} now gives
\begin{equation*}
{\rm (VI)}\le L^2 \eta^2 (2\tau^2 B_g^2 +{8 \max_{r,t}\|b^{r,t}\|^2 d nV^2}).  
\end{equation*}
To handle the term (VII), we use the $L$-smoothness to bound the gradient norm with loss value suboptimality. To this end, we start with the following inequality
\begin{equation*}
f_i^{r,t}(y)\le f_i^{r,t}(x)+\langle \nabla f_i^{r,t}(x), y-x \rangle+\frac{L}{2}\|y-x\|^2.  
\end{equation*}
By averaging the above inequality over $t$ and $i$ and optimizing both sides of the resulting inequality w.r.t. $y$, 
we get 
\begin{equation*}
\begin{aligned}
&(f^{r})^{\star}\le \min_y	\frac{1}{n\tau}\sum_{i=1}^n\sum_{t=0}^{\tau-1} f_i^{r,t}(y)\\
\leq\ & \frac{1}{n\tau}\sum_{i=1}^n\sum_{t=0}^{\tau-1}  f_i^{r,t}(x)- \frac{1}{2L}\left\|\frac{1}{n\tau}\sum_{i=1}^n\sum_{t=0}^{\tau-1} \nabla f_i^{r,t}(x) \right\|^2\\
& +\frac{L}{2}\min_{y}\left\{\left\| y-\left(x-  \frac{1}{n\tau L}\sum_{i=1}^n\sum_{t=0}^{\tau-1} \nabla f_i^{r,t}(x) \right) \right\|^2\right\}\\
=\ & \frac{1}{n\tau}\sum_{i=1}^n\sum_{t=0}^{\tau-1} f_i^{r,t}(x)- \frac{1}{2L}\left\|\frac{1}{n\tau}\sum_{i=1}^n\sum_{t=0}^{\tau-1} \nabla f_i^{r,t}(x) \right\|^2,
\end{aligned} 
\end{equation*}
which implies that
\begin{equation*}
\begin{aligned}
\Big\|\frac{1}{n\tau}\sum_{i=1}^n\sum_{t=0}^{\tau-1}\nabla  f_i^{r,t}(x)\Big\|^2\!\le 2L\Big(\frac{1}{n\tau}\sum_{i=1}^n\sum_{t=0}^{\tau-1} f_i^{r,t}(x)\!-\! (f^{r})^{\star} \Big).
\end{aligned}    
\end{equation*}
Substituting this inequality into (VII), we have 
\begin{equation*}
\begin{aligned}
{\rm (VII)}
&= \bE \Big\| \frac{1}{n\tau} \sum_{i=1}^n\sum_{t=0}^{\tau-1}\nabla f_i^{r,t}(x^r)\Big\|^2\\
&\le \frac{1}{n\tau} \sum_{i=1}^n\sum_{t=0}^{\tau-1} 2L \bE( f_i^{r,t}(x^r) -(f^{r})^{\star}  ).
\end{aligned}    
\end{equation*}
Next, we combine the derived upper bounds for (VI) and (VII), the expressions for $c_{\rm VI}$ and $c_{\rm VII}$ and~\eqref{eq-16} to find
\begin{align}
&\bE\|x_{\xi}^{r+1}-\operatorname{P}_{\mathcal{X}^{\star}}^{x_{\xi}^{r+1}}\|^2 -\bE\|x_{\xi}^{r}-\operatorname{P}_{\mathcal{X}^{\star}}^{x_{\xi}^{r}}\|^2\label{eq-a17} \\\notag
\le& \underbrace{\left( \teta^2\left(1+\frac{3\theta}{2\teta}\right)4L -{2\teta{\alpha}}\right)}_{\le -\teta{\alpha}} \frac{1}{n\tau}\sum_{i=1}^n\sum_{t=0}^{\tau-1} \bE(f_i^{r,t}(x^r)-(f^{r})^{\star})\\\notag
&+\left( \teta^2 \left(1+\frac{\theta}{\teta}\right) + \frac{\teta}{\mu}\right)2L^2\eta^2 (2\tau^2 B_g^2 +{8 \max_{r,t}\|b^{r,t}\|^2 d nV^2})   \\\notag
&+ \frac{2\teta}{\theta}\left(\bE\|x_{\xi}^r-x^r\|^2+\bE\|\operatorname{P}_{\mathcal{X}^{\star}}^{x^{r}}-\operatorname{P}_{\mathcal{X}^{\star}}^{x_{\xi}^{r}}\|^2\right)\\\notag
&+\teta\left(\frac{1}{\theta}+\mu\right)\bE\|\operatorname{P}_{\mathcal{X}_r^{\star}}^{x^{r}}-\operatorname{P}_{\mathcal{X}^{\star}}^{x^{r}}\|^2.
\end{align}       
We now need to choose $\theta$ to guarantee that $ \left( \teta^2\left(1+\frac{3\theta}{2\teta}\right)4L-{2\teta{ \alpha}} \right)\le -\teta{ \alpha} $, which holds if 
\begin{equation}\label{eqtheta}
\teta + \frac{3}{2}\theta \le \frac{{\alpha}}{4L}.    
\end{equation}
The condition \eqref{eqtheta} will be checked later when we choose the specific algorithm parameters. 

Dividing both sides of \eqref{eq-a17} by $\teta$, we obtain 
\begin{align}\label{eq-24}
&\frac{1}{\teta}\bE\|x_{\xi}^{r+1}-\operatorname{P}_{\mathcal{X}^{\star}}^{x_{\xi}^{r+1}}\|^2 -\frac{1}{\teta}\bE\|x_{\xi}^{r}-\operatorname{P}_{\mathcal{X}^{\star}}^{x_{\xi}^{r}}\|^2 \\ \notag
\le & -  \frac{{\alpha}}{n\tau}\sum_{i=1}^n\sum_{t=0}^{\tau-1} \bE(f_i^{r,t}(x^r)-(f^{r})^{\star})\\\notag
&+\left( \teta \left(1+\frac{\theta}{\teta}\right) + \frac{1}{\mu}\right)2L^2\eta^2 (2\tau^2 B_g^2 +{8 \max_{r,t}\|b^{r,t}\|^2 d nV^2})\\\notag
&+ \frac{2}{\theta}\left(\bE\|x_{\xi}^r-x^r\|^2+\bE\|\operatorname{P}_{\mathcal{X}^{\star}}^{x^{r}}-\operatorname{P}_{\mathcal{X}^{\star}}^{x_{\xi}^{r}}\|^2\right)\\\notag
&+\left(\frac{1}{\theta}+\mu\right)\bE\|\operatorname{P}_{\mathcal{X}_r^{\star}}^{x^{r}}-\operatorname{P}_{\mathcal{X}^{\star}}^{x^{r}}\|^2.\notag
\end{align}         
For notational convenience, we dote the last three terms in the right-hand of~\eqref{eq-24} by $S^r$. We then rewrite the results as
\begin{equation}\label{eq-25}
\begin{aligned}
&\frac{{\alpha}}{n\tau}\sum_{i=1}^n\sum_{t=0}^{\tau-1} \bE(f_i^{r,t}(x^r)-(f^{r})^{\star})\\
\le\ & \frac{1}{\teta} \bE\|x_{\xi}^{r}-\operatorname{P}_{\mathcal{X}^{\star}}^{x_{\xi}^{r}}\|^2 -  \frac{1}{\teta}\bE\|x_{\xi}^{r+1}-\operatorname{P}_{\mathcal{X}^{\star}}^{x_{\xi}^{r+1}}\|^2 + S^r.    
\end{aligned}    
\end{equation}
Repeated application of~\eqref{eq-25} and use of Assumption~\ref{asm-regular} gives 
\begin{equation}\label{eq-26}
\begin{aligned}
& \frac{1}{R}\sum_{r=0}^{R-1}\frac{1}{n\tau}\sum_{i=1}^n\sum_{t=0}^{\tau-1} \bE(f_i^{r,t}(x^r)-(f^{r})^{\star})\\  
\le\ & \frac{\bE\|x_{\xi}^0-\operatorname{P}_{\mathcal{X}^{\star}}^{x_{\xi}^{0}}\|^2}{{\alpha}R\teta}  + \frac{\sum_{r=0}^{R-1} S^r}{{\alpha}R}.
\end{aligned}    
\end{equation}
Substituting $x_{\xi}^r-x^r=\frac{\teta}{\tau}b^{r-1,\tau-1}\bxi$ and Assumption \ref{asm-regular} yields
\begin{align}\label{eq-27}
&\frac{1}{R}\sum_{r=0}^{R-1} S^r\\\notag
\le& \left( \teta \left(1+\frac{\theta}{\teta}\right) + \frac{1}{\mu}\right)2L^2\eta^2(2\tau^2 B_g^2 +{8 \max_{r,t}\|b^{r,t}\|^2 d nV^2})\\\notag
+& \!\frac{1}{\theta R}\!\sum_{r=0}^{R-1}(2(1\!+\!\sigma)\bE\| {\frac{\teta}{\tau}}b^{r-1,\tau-1}\bxi\|^2+(1\!+\!\theta\mu)\bE\|\operatorname{P}_{\mathcal{X}_r^{\star}}^{x^{r}}\!-\!\operatorname{P}_{\mathcal{X}^{\star}}^{x^{r}}\|^2).
\end{align}     
   
Using the fact that 
\begin{equation}\label{eq-B-xi}
\bE \|b^{r-1,\tau-1}\bxi\|^2= \|b^{r-1,\tau-1}\|^2 dV^2,
\end{equation}
we can use~\eqref{eq-26} to bound the regret as
\begin{equation*}
\begin{aligned}
& \frac{\reg}{R\tau}:=\frac{1}{R}\sum_{r=0}^{R-1}\frac{1}{n\tau}\sum_{i=1}^n\sum_{t=0}^{\tau-1} \bE(f_i^{r,t}(x^r)-(f^{r})^{\star})\\  
\le& \frac{\bE\|x_{\xi}^0-\operatorname{P}_{\mathcal{X}^{\star}}^{x_{\xi}^{0}}\|^2}{{\alpha}R\teta} \\
+& \frac{2(1+\sigma)\teta^2}{{\alpha}\theta R} {\|  \bB_R \|_F^2} {\frac{dV^2}{\tau^2}} \!+\!\frac{1}{{\alpha}\theta R}\sum_{r=0}^{R-1}(1+\theta\mu)\bE\|\operatorname{P}_{\mathcal{X}_r^{\star}}^{x^{r}}-\operatorname{P}_{\mathcal{X}^{\star}}^{x^{r}}\|^2\\
+& {\frac{1}{ \alpha}} \left( \teta \left(1+\frac{\theta}{\teta}\right) + \frac{1}{\mu}\right)2L^2\eta^2(2\tau^2 B_g^2 +{8 \max_{r,t}\|b^{r,t}\|^2 d nV^2}),
\end{aligned}      
\end{equation*}
where $\|\bB_R\|_F^2:=\|b^{0,\tau-1}\|^2+\cdots+\|b^{R-1,\tau-1}\|^2$.
{Due to the last two terms in the above inequality, we need a large $\theta$; but $\theta$ must satisfy condition \eqref{eqtheta}. Inspired by this, taking half of the upper bound $\alpha/(4L)$ in condition \eqref{eqtheta} as $\frac{3}{2}\theta$, then we get}
\begin{equation*}
\theta=\frac{{\alpha}}{12L} \quad \text{and} \quad\teta\le \frac{{\alpha}}{8L}.
\end{equation*} 
Substituting $\theta=\frac{{\alpha}}{12L}$, then we have 
\begin{align}\label{eq-utility}
&\frac{\reg}{R\tau}  
\le \frac{\bE\|x_{\xi}^0-\operatorname{P}_{\mathcal{X}^{\star}}^{x_{\xi}^{0}}\|^2}{{\alpha}R\teta}\\\notag
&+\frac{24L(1+\sigma)\teta^2}{ \alpha^2 R} {\|\bB_R \|_F^2}  {\frac{dV^2}{\tau^2}}+
\frac{\left(12L+{\alpha}{\mu}\right)C_{R}}{{ \alpha^2R}}\\\notag
&+ {\frac{1}{\alpha}} \left( \teta+\frac{{\alpha}}{12L} + \frac{1}{\mu}\right)2L^2\eta^2(2\tau^2 B_g^2 +{8 \max_{r,t}\|b^{r,t}\|^2 d nV^2}),
\end{align}    
where  $C_{R}:=\sum_{r=0}^{R-1}\bE\|\operatorname{P}_{\mathcal{X}_r^{\star}}^{x^{r}}-\operatorname{P}_{\mathcal{X}^{\star}}^{x^{r}}\|^2$. This completes the proof of Lemma~\ref{lem-u}.

\subsection{Proof of Theorem~\ref{thm-up}}\label{appthm-up}
In the following, we substitute the specific parameter values given by our DP analysis and choose the step sizes to simplify \eqref{eq-utility}. Substituting the fact given in \eqref{eq-exmp-iii} that 
\begin{equation*}
\begin{aligned}
&\max_{r,t}\|b^{r,t}\|^2\le \mathcal{O}(\ln (R\tau)),  \\
&\|\bB_R\|_F^2:=\|b^{0,\tau-1}\|^2+\cdots+\|b^{R-1,\tau-1}\|^2\le \mathcal{O}(R\ln (R\tau))   
\end{aligned}    
\end{equation*}
into~\eqref{eq-utility}, we have
\begin{equation*}
\begin{aligned}
\frac{\reg}{R\tau} 
\!\le\! \mathcal{O} \left(\! \frac{1}{R\teta}  + \frac{\teta^2}{\eta_g^2} B_g^2 + \left(1\!+\!\frac{n}{\eta_g^2}\right) \teta^{2} \ln (R\tau) \frac{dV^2}{{\tau^2}}  + \frac{C_{R}}{R}\!\right),     
\end{aligned}  
\end{equation*}
where we use $\teta=\eta\eta_g\tau$. 
Next, we use the bound 
\begin{equation}
V^2=\frac{V_i^2}{n}\le \mathcal{O} \left({\ln (R\tau)}  \frac{B_g^2}{n}  \frac{(\ln \frac{1}{\delta}) }{\epsilon^2} \right)      
\end{equation}
to get
\begin{equation}\label{eq-27-1}
\begin{aligned}
&\frac{\reg}{R\tau} \le \ \mathcal{O} \left( \frac{1}{R\teta}  + \frac{\teta^2}{\eta_g^2} B_g^2 +  \frac{C_{R}}{R} \right.\\
&\left.+ \left(1+\frac{n}{\eta_g^2}\right)\teta^{2} \ln (R\tau)  d{\ln (R\tau)} \frac{B_g^2}{n}  \frac{(\ln \frac{1}{\delta}) }{{\tau^2}\epsilon^2} \right).    
\end{aligned}    
\end{equation}
This completes the proof of \eqref{eq-27-11}. 

Let $\mathcal{O} (\frac{1}{R\teta })= \mathcal{O}(\teta^2 (\ln (R\tau))^2)$ and consider condition \eqref{eqtheta}, then we have $\teta=\mathcal{O}\left({\min}\left\{  \mathcal{O}( R^{-\frac{1}{3}} (\ln (R\tau))^{-\frac{2}{3}}),\frac{\alpha}{8L}\right\} \right) = \mathcal{O}\left( R^{-\frac{1}{3}} (\ln (R\tau))^{-\frac{2}{3}}\right)$ and derive 
\begin{equation*}
\begin{aligned}
\frac{\reg}{R\tau}\le\ &\mathcal{O}\left( \frac{(\ln (R\tau))^{\frac{2}{3}}}{R^{\frac{2}{3}}} +\frac{B_g^2}{\eta_g^2 R^{\frac{2}{3}}(\ln (R\tau))^{\frac{4}{3}}} \right.\\
&\left.+ \left(1+\frac{n}{\eta_g^2}\right) \frac{(\ln (R\tau))^{\frac{2}{3}}}{R^{\frac{2}{3}}}\frac{d B_g^2(\ln \frac{1}{\delta})}{n {\tau^2} \epsilon^2}+\frac{C_{R}}{R}\right).    
\end{aligned}    
\end{equation*}
This completes the proof of Theorem~\ref{thm-up}.

\clearpage
\twocolumn[
\begin{@twocolumnfalse}
	\section*{\centering{Supplementary Material for \\ \emph{ Locally Differentially Private Online Federated Learning With Correlated Noise	\\[15pt]}}}
\end{@twocolumnfalse}
]
\setcounter{subsection}{0}

{ 
\subsection{Proof of the equivalence between \eqref{eq-comp-alg} and Algorithm~\ref{alg-fl}}\label{app-equi-compact}
\begin{proof}
With Line 16 and Line 13 of Algorithm \ref{alg-fl}, we have
\begin{equation*}
\begin{aligned}
x^{r+1}&=x^r-\teta \frac{1}{n}\sum_{i=1}^n \left(\frac{1}{\tau} \sum_{t=0}^{\tau-1} \hat{\nabla} f_i^{r,t} \right) \\
&=x^r-\teta \frac{1}{n}\sum_{i=1}^n \frac{1}{\tau} \sum_{t=0}^{\tau-1} \left(S_{i}^{r,t}-S_i^{r,t-1} \right)\\
&= x^r-\teta \frac{1}{n}\sum_{i=1}^n \frac{1}{\tau} \sum_{t=0}^{\tau-1} \left(\nabla f_i^{r,t}(z_i^{r,t})+( b^{r,t}-b^{r,t-1})\bxi_i \right)\\
&=  x^r-\teta \left(g^r+\frac{1}{\tau} \sum_{t=0}^{\tau-1}( b^{r,t}-b^{r,t-1}) \bxi\right)\\
&=  x^r-\teta \left(g^r+\frac{1}{\tau}( b^{r,\tau-1}-b^{r-1,\tau-1})\bxi\right),
\end{aligned}    
\end{equation*}  
where the second equality is due to Line 10 in Algorithm \ref{alg-fl}, the third equality is due to Line 9 in Algorithm \ref{alg-fl}, the fourth equality is due to the definition of $g^r$ and $\bxi$, and the last equality is due to $b^{r,-1}=b^{r-1,\tau-1}$.
\end{proof}
}
\vspace{-2cm}
\subsection{Proof of Corollary~\ref{lem-u-IID}: independent noise}\label{app-lem-u-IID}

With independent noise, the algorithm becomes
\begin{align}\label{eqn:indp}
&z_i^{r,t+1} =z_i^{r,t}-\eta(\nabla f_i^{r,t}(z_i^{r,t})+{\xi_i^{r,t} }), \\ \notag
&x^{r+1}=x^r-\teta \Bigg( \underbrace{\frac{1}{n\tau} \sum_{t=0}^{\tau-1}\sum_{i=1}^n \nabla f_i^{r,t}(z_i^{r,t})}_{:=g^r} + \underbrace{\frac{1}{n\tau} \sum_{t=0}^{\tau-1}\sum_{i=1}^n  \xi_i^{r,t}}_{:=\xi^r}\Bigg),
\end{align}
where $\xi_i^{r,t}\sim \mathcal{N}(0, V_i^2)^{1\times d}$. 
In this way,~\eqref{eq-cvx-r+1-r} simplifies to
\begin{equation}\label{eq-cvx-r+1-r-indp}
\begin{aligned}
&\bE\|x^{r+1}-\operatorname{P}_{\mathcal{X}^{\star}}^{x^{r+1}}\|^2\\
\leq\ &\bE \|x^r-\operatorname{P}_{\mathcal{X}^{\star}}^{x^{r}}-\teta (g^r+\xi^r) \|^2 \\ 
\le\ & \bE\|x^r-\operatorname{P}_{\mathcal{X}^{\star}}^{x^{r}} \|^2+ 2\teta^2d\frac{V^2}{{\tau}} \\
&\underbrace{-2\teta \bE\left\langle g^r,x^r-\operatorname{P}_{\mathcal{X}^{\star}}^{x^{r}}\right\rangle}_{\rm (IV)}+2\teta^2\underbrace{\bE\|g^r\|^2}_{\rm (V)},
\end{aligned} 
\end{equation}
where we use  $\bE[\langle\xi^r, x^r-\operatorname{P}_{\mathcal{X}^{\star}}^{x^{r}}\rangle|\mathcal{F}^r]=0$ where $\mathcal{F}^r$ is the sigma-algebra generated by all the randomness up to communication round $r$,  $\bE\|\xi^r\|^2\le \frac{1}{n\tau} V_i^2$ and $V^2=V_i^2/n$ for all $i$. 

Following a similar strategy to bound $\rm (IV)$ and $\rm (V)$ as before, \eqref{eq-16} becomes 
\begin{equation}\label{eq-16-1}
\begin{aligned}
&\bE\|x^{r+1}-\operatorname{P}_{\mathcal{X}^{\star}}^{x^{r+1}}\|^2 -\bE\|x^{r}-\operatorname{P}_{\mathcal{X}^{\star}}^{x^{r}}\|^2 \\
&\le-\frac{2\teta{\alpha}}{n\tau}\sum_{i=1}^n\sum_{t=0}^{\tau-1} \bE(f_i^{r,t}(x^r)-(f^{r})^\star)+2\teta^2d\frac{V^2}{{\tau}}\\
&+2c_{{\rm VI}}\cdot\underbrace{\bE\Big\| \frac{1}{n\tau} \sum_{i=1}^n\sum_{t=0}^{\tau-1} (\nabla f_i^{r,t}(z_i^{r,t})-\nabla f_i^{r,t}(x^r))\Big\|^2}_{\rm (VI)}\\
&+2c_{{\rm VII}}\cdot \underbrace{\bE \left\| \frac{1}{n\tau} \sum_{i=1}^n\sum_{t=0}^{\tau-1}\nabla f_i^{r,t}(x^r)
\right\|^2}_{\rm (VII)}\\
&+\teta\left(\frac{1}{\theta}+\mu\right)\bE\|\operatorname{P}_{\mathcal{X}_r^{\star}}^{x^{r}}-\operatorname{P}_{\mathcal{X}^{\star}}^{x^{r}}\|^2,
\end{aligned}       
\end{equation}
where   $c_{{\rm VI}}:=2\teta^2 + \frac{\teta}{\mu}$ and $c_{{\rm VII}}:=\teta^2\left(2+\frac{\theta}{2\teta}\right)$. 

By repeated  application of the local updates, we thus have
{  
$z_i^{r,t}=z_{i}^{r,0}-\eta (\nabla f_i^{r,0}(z_i^{r,0})+\cdots+\nabla f_i^{r,t-1}(z_i^{r,t-1}) + \xi_i^{r,0}+\cdots+\xi_i^{r,t-1} )$} for all $t$.  By substituting $z_i^{r,0}=x^r$, 
$\|\nabla f_i^{r,t}(x)\|\le B_g$ for any $x$, we have 
\begin{equation*}
\begin{aligned}
\bE\|z_i^{r,t}-x^r\|^2=&\ \eta^2\bE \Big\|  \sum_{\tilde{t}=0}^{t-1} \nabla f_i^{r,\tilde{t}}(z_i^{r,\tilde{t}}) + {\xi_i^{r,0}+\cdots+\xi_i^{r,t-1}} \Big\|^2 \\
\le&\ \eta^2 (2\tau^2 B_g^2 +{2 \tau d nV^2}),   
\end{aligned}    
\end{equation*}
{where we use that $\xi_i^{r,t}$ are IID for all $r,t$, $ \bE\|\xi_i^{r,0}+\cdots+\xi_i^{r,t-1}\|^2=\bE\|\xi_i^{r,0}\|^2+\cdots+\bE\|\xi_i^{r,t-1}\|^2 \le \tau d V_i^2 $, and $V_i^2=nV^2$  in the last inequality.}
With the above inequality, the terms $\rm (VI)$ and $\rm (VII)$ can be bounded similarly as before and \eqref{eq-a17} becomes
\begin{equation}\label{eq-a17-indp}
\begin{aligned}
&\bE\|x^{r+1}-\operatorname{P}_{\mathcal{X}^{\star}}^{x^{r+1}}\|^2 -\bE\|x^{r}-\operatorname{P}_{\mathcal{X}^{\star}}^{x^{r}}\|^2 \\ \notag
\le& \underbrace{\left(\teta^2\left(4+\frac{\theta}{\teta}\right)2L -2\teta{\alpha}\right)}_{\le -\teta{\alpha}/2} \frac{1}{n\tau}\sum_{i=1}^n\sum_{t=0}^{\tau-1} \bE(f_i^{r,t}(x^r)-(f^{r})^{\star})\\
&+\left( 2\teta^2 + \frac{\teta}{\mu}\right)2L^2\eta^2 (2\tau^2 B_g^2 +{2 \tau d nV^2}) \\ \notag
&+2\teta^2 d \frac{V^2}{{\tau}} 
+\teta\left(\frac{1}{\theta}+\mu\right)\bE\|\operatorname{P}_{\mathcal{X}_r^{\star}}^{x^{r}}-\operatorname{P}_{\mathcal{X}^{\star}}^{x^{r}}\|^2,
\end{aligned} 
\end{equation}
if we require that 
\begin{equation}\label{eqtheta-iid}
2\teta + \frac{1}{2}\theta \le \frac{{3\alpha}}{8L}.
\end{equation}
This condition can be satisfied by setting 
$\theta=\frac{{\alpha}}{12L}$ and $\teta\le \frac{{\alpha}}{8L} $. 

Moreover, with
\begin{equation}\label{eq-27-indp}
\begin{aligned}
\frac{1}{R}\sum_{r=0}^{R-1} S^r
\le&\ \frac{1}{{\alpha}}\left( 2\teta + \frac{1}{\mu}\right)4L^2\eta^2(\tau^2 B_g^2 +{ \tau d nV^2})+\frac{2\teta d}{{\alpha}} {\frac{V^2}{\tau} }\\
&+ \frac{1}{{\alpha} R}\sum_{r=0}^{R-1}\left(\frac{1}{\theta}+\mu\right) \bE\|\operatorname{P}_{\mathcal{X}_r^{\star}}^{x^{r}}-\operatorname{P}_{\mathcal{X}^{\star}}^{x^{r}}\|^2,
\end{aligned}    
\end{equation}
we get the regret bound
\begin{equation*}
\begin{aligned}
\frac{\reg}{R\tau}  
\le\ & \frac{\bE\|x^0-\operatorname{P}_{\mathcal{X}^{\star}}^{x^{0}}\|^2}{{\alpha}R\teta} +\frac{2\teta d}{\alpha}{{\frac{V^2}{\tau}}}+\frac{\left(12L+{\alpha\mu}\right)C_{R}}{\alpha^2 R}\\
&+\frac{1}{\alpha}\left( 2\teta + \frac{1}{\mu}\right)4L^2\eta^2(\tau^2 B_g^2 +{ \tau d nV^2}).
\end{aligned}    
\end{equation*}
With the privacy analysis in Lemma \ref{lem-priv-analy}, we know that 
\begin{equation}
V^2= V_i^2/n =  \mathcal{O} \left(  \frac{B_g^2}{n}  \frac{(\ln \frac{1}{\delta}) }{\epsilon^2} \right).  
\end{equation}
Substituting this expression in our regret bound yields
\begin{equation*}
\begin{aligned}
\frac{\reg}{R\tau} 
\le\mathcal{O}\!\left( \frac{1}{R\teta} \! +\! \frac{\teta^2}{\eta_g^2} B_g^2 \!+\! \Big(1+\frac{\teta n}{\eta_g^2}\Big) {\teta} d  \frac{B_g^2}{n}  \frac{(\ln \frac{1}{\delta}) }{{\tau}\epsilon^2} \! +\! \frac{C_{R}}{R}\right).    
\end{aligned}    
\end{equation*}
Comparing the above inequality for independent noise with \eqref{eq-27-1} for correlated noise, 
if the step size $\teta$ satisfies 
\begin{equation*}
\teta\le   \frac{\tau}{\left(1+\frac{n}{\eta_g^2}\right) (\ln(R\tau))^2},  
\end{equation*}
correlated noise is provably better than independent noise. 

Let $\mathcal{O}(\frac{1}{R\teta })= \mathcal{O}(\teta) $ and consider condition \eqref{eqtheta-iid}, then $\teta= \mathcal{O}\left(\min\left\{\mathcal{O}( R^{-\frac{1}{2}}), \frac{\alpha}{8L} \right\}\right) = \mathcal{O}( R^{-\frac{1}{2}})$ and  
\begin{equation*}
\begin{aligned}
\frac{\reg}{R\tau}\le\mathcal{O}\left( \frac{1}{R^{\frac{1}{2}}} +\frac{B_g^2}{\eta_g^2 R} + \frac{1}{R^{\frac{1}{2}}}\frac{d B_g^2(\ln \frac{1}{\delta})}{n{\tau}\epsilon^2}+\frac{C_{R}}{R}\right).    
\end{aligned}    
\end{equation*}
This completes the proof of Corollary~\ref{lem-u-IID}.

\subsection{Dynamic regret under SC}\label{app-sc-dynamic}

The proof in this section is auxiliary and serves two purposes: i) to prepare for the proof of Corollary~\ref{coro} on static regret under SC in the next section, and ii) to provide a comparative result, demonstrating that even under SC, using dynamic regret does not improve the upper bound on regret with respect to its dependence on $C_R$.

For the SC case, we know from~\eqref{eq-cvx-r+1-r} and the fact that $x^\star=\operatorname{P}_{\mathcal{X}^{\star}}^{x^{r}}=\operatorname{P}_{\mathcal{X}^{\star}}^{x_{\xi}^{r}}$ that 
\begin{equation*}
\begin{aligned}
&\bE\|x_{\xi}^{r+1}-x^\star\|^2\\
\le\ & \bE\|x_{\xi}^r-x^\star \|^2 \underbrace{-2\teta \bE\left\langle g^r,x^r-x^\star\right\rangle}_{\rm (IV)}+\teta^2\left(1+\frac{\theta}{\teta}\right)\underbrace{\bE\|g^r\|^2}_{\rm (V)}\\ 
&+ \frac{\teta}{\theta}\bE\|x_{\xi}^r-x^r\|^2.
\end{aligned} 
\end{equation*}

{We use a similar decomposition of (IV) as in \eqref{eq-IV-decomp} and denote the two terms obtained as (IV.I) and (IV.II), respectively.}
The bound for (IV.I) can be simplified to
\begin{equation*}
{\rm (IV.I)}
\le \! -\frac{2\teta}{n\tau}\sum_{i=1}^n\sum_{t=0}^{\tau-1} \bE(f_i^{r,t}(x^r)-f_i^{r,t}(x^\star)-{\teta\mu}\bE\|x^r-x^\star\|^2,
\end{equation*}
where we have used that strong convexity (SC) implies 
{\begin{equation*}
\begin{aligned}
&\frac{1}{n\tau}\sum_{i=1}^n\sum_{t=0}^{\tau-1} (f_i^{r,t}(x)-f_i^{r,t}(x^\star))+\frac{\mu}{2}\|x-x^\star\|^2\\
\le& \left\langle \frac{1}{n\tau}\sum_{i=1}^n\sum_{t=0}^{\tau-1} \nabla f_i^{r,t}(x), x-x^\star\right\rangle, ~\forall x.
\end{aligned}    
\end{equation*}}
Using the corresponding simplified (IV.II), we find 
\begin{equation*}
\begin{aligned}
{\rm (IV)}\le & 
-\frac{2\teta}{n\tau}\sum_{i=1}^n\sum_{t=0}^{\tau-1} \bE(f_i^{r,t}(x^r)-f_i^{r,t}(x^\star))\\
& +\frac{\teta}{\mu}\bE \Big\|\frac{1}{n\tau}\sum_{t=0}^{\tau-1}\sum_{i=1}^n (\nabla f_{i}^{r,t}(z_{i}^{r,t})-\nabla f_{i}^{r,t}(x^{r})) \Big\|^2.
\end{aligned}    
\end{equation*}
Following a similar analysis as before, \eqref{eq-16} becomes 
\begin{align}
&\bE\|x_{\xi}^{r+1}-x^{\star}\|^2 -\bE\|x_{\xi}^{r}-x^{\star}\|^2 \label{eq-16-sc}\\ \notag
&\le-\frac{2\teta}{n\tau}\sum_{i=1}^n\sum_{t=0}^{\tau-1} \bE(f_i^{r,t}(x^r)-f_i^{r,t}(x^\star) \underbrace{-(f^r)^{\star} + (f^r)^{\star})}_{\text{for dynamic Regret}}\\ \notag
&+2c_{{\rm VI} }\cdot\underbrace{\bE\Big\| \frac{1}{n\tau} \sum_{i=1}^n\sum_{t=0}^{\tau-1} (\nabla f_i^{r,t}(z_i^{r,t})-\nabla f_i^{r,t}(x^r))\Big\|^2}_{{\rm (VI)} \le L^2\eta^2(2\tau^2 B_g^2 +{8 \max_{r,t}\|b^{r,t}\|^2 d nV^2}) } \\\notag
&+2c_{{\rm VII}}\cdot \underbrace{\bE \left\| \frac{1}{n\tau} \sum_{i=1}^n\sum_{t=0}^{\tau-1}\nabla f_i^{r,t}(x^r)
\right\|^2}_{{\rm (VII)}\le\frac{1}{n\tau} \sum_{i=1}^n\sum_{t=0}^{\tau-1} 2L \bE( f_i^{r,t}(x^r) -(f^{r})^{\star}  )}\\ \notag
&+ \frac{\teta}{\theta}\bE\|x_{\xi}^r-x^r\|^2.
\end{align}       
If \eqref{eqtheta} holds, following our earlier analysis, we have
\begin{equation}\label{eq-sc-dynm}
\begin{aligned}
&\frac{1}{\teta}\bE\|x_{\xi}^{r+1}-x^\star\|^2 -\frac{1}{\teta}\bE\|x_{\xi}^{r}-x^\star\|^2 \\
&\le-  \frac{1}{n\tau}\sum_{i=1}^n\sum_{t=0}^{\tau-1} \bE(f_i^{r,t}(x^r)-(f^r)^{\star})\\
&+\left( \teta \left(1+\frac{\theta}{\teta}\right) + \frac{1}{2\mu}\right)2L^2\eta^2(2\tau^2 B_g^2 +{8 \max_{r,t}\|b^{r,t}\|^2 d nV^2})\\
&+ \frac{1}{\theta}\bE\| \frac{\teta}{{\tau}}b^{r-1,\tau-1} \bxi\|^2
+\underbrace{2 (f^r(x^{\star}) -(f^r)^{\star}  )}_{\le {L}\|x^{\star} -x_r^{\star}\|^2 }.
\end{aligned}         
\end{equation}
Set $\theta={1}/{(12L)}$ and $\teta\le {1}/{(8L)}$ so that condition~\eqref{eqtheta} holds, and 
let $C_R:= \sum_{r=0}^{R-1}\bE\|x^{\star}-x_r^{\star}\|^2$. Then 
$$
\begin{aligned}
&\frac{1}{R}\sum_{r=0}^{R-1}S^r\\
\le &\ \left( \teta \left(1+\frac{\theta}{\teta}\right) + \frac{1}{2\mu}\right)2L^2\eta^2(2\tau^2 B_g^2 +{8 \max_{r,t}\|b^{r,t}\|^2 d nV^2})\\
&+\frac{{\|\bB_R\|_F^2}}{\theta R}\frac{\teta^2}{{\tau^2}} dV^2+\frac{LC_R}{R}\\
\le&\ \mathcal{O} \left( \frac{\teta^2}{\eta_g^2} B_g^2  + \left(1+\frac{n}{\eta_g^2}\right)\frac{\teta^2\ln(R\tau) dV^2}{\tau^2} + \frac{C_{R}}{R}\right), 
\end{aligned}
$$
where the last inequality is due to $\max_{r,t}\|b^{r,t}\|^2\le \mathcal{O}(\ln(R\tau))$ and $\|\bB_R\|_F^2\le \mathcal{O}(R\ln(R\tau))$. 
We find the regret bound
\begin{equation}
\begin{aligned}
& \frac{\reg}{R\tau} 
\le \mathcal{O} \left( \frac{1}{R\teta}  + \frac{\teta^2}{\eta_g^2} B_g^2\right. \\
& \left.+ \left(1+\frac{n}{\eta_g^2}\right) \frac{\teta^{2}}{{\tau^2}} {\ln (R\tau)} d \ln (R\tau)  \frac{B_g^2}{n}  \frac{(\ln \frac{1}{\delta}) }{\epsilon^2}  + \frac{C_{R}}{R}\right).    
\end{aligned}    
\end{equation}

Similarly as before, with $\mathcal{O}\left(\frac{1}{R\teta }\right)= \mathcal{O}\left(\teta^2 (\ln (R\tau))^2 \right)$ and $\teta\le1/(8L)$, we have $\teta= \mathcal{O}( R^{-\frac{1}{3}} (\ln (R\tau))^{-\frac{2}{3}})$ and 
\begin{equation}\label{eq-sc-dyn}
\begin{aligned}
&\frac{\reg}{R\tau}\le\mathcal{O}\left( \frac{(\ln (R\tau))^{\frac{2}{3}}}{R^{\frac{2}{3}}} +\frac{B_g^2}{\eta_g^2 R^{\frac{2}{3}}(\ln (R\tau))^{\frac{4}{3}}} \right.\\
&\left.+\left(1+\frac{n}{\eta_g^2}\right) \frac{(\ln (R\tau))^{\frac{2}{3}}}{R^{\frac{2}{3}}}\frac{d B_g^2(\ln \frac{1}{\delta})}{n{\tau^2} \epsilon^2}+\frac{C_{R}}{R}\right).    
\end{aligned}    
\end{equation}
Hence, even under SC, the dependence on $C_R$ is ${C_{R}}/{R}$. 
\subsection{ Proof of Corollary~\ref{coro}: Static Regret under SC}\label{app-sc-static}
In the following, we show that the dependence on $C_R$ can be improved if we use static regret. With static regret, if condition~\eqref{eqtheta} holds, \eqref{eq-16-sc} becomes 
\begin{equation*}
\begin{aligned}
&\bE\|x_{\xi}^{r+1}-x^{\star}\|^2 -\bE\|x_{\xi}^{r}-x^{\star}\|^2 \\
&\le-\frac{2\teta}{n\tau}\sum_{i=1}^n\sum_{t=0}^{\tau-1} \bE(f_i^{r,t}(x^r)-f_i^{r,t}(x^\star))\\
&+2c_{{\rm VI} }\cdot\underbrace{\bE\Big\| \frac{1}{n\tau} \sum_{i=1}^n\sum_{t=0}^{\tau-1} (\nabla f_i^{r,t}(z_i^{r,t})-\nabla f_i^{r,t}(x^r))\Big\|^2}_{ \le L^2\eta^2(2\tau^2 B_g^2 +{8 \max_{r,t}\|b^{r,t}\|^2 d nV^2}) } \\
&+2c_{{\rm VII}}\cdot \underbrace{\bE \left\| \frac{1}{n\tau} \sum_{i=1}^n\sum_{t=0}^{\tau-1}\nabla f_i^{r,t}(x^r)
\right\|^2}_{\hspace{-1cm}\le\frac{1}{n\tau} \sum_{i=1}^n\sum_{t=0}^{\tau-1} 2L \bE( f_i^{r,t}(x^r) -(f^{r})^{\star} \underbrace{+ f^r(x^{\star}) -f^r(x^{\star})  }_{\text{for static Regret}}) }\\
&+ \frac{\teta}{\theta}\bE\|x_{\xi}^r-x^r\|^2
\end{aligned}       
\end{equation*}
and~\eqref{eq-sc-dynm} evaluates to 
\begin{equation*}
\begin{aligned}
&\frac{1}{\teta}\bE\|x_{\xi}^{r+1}-x^\star\|^2 -\frac{1}{\teta}\bE\|x_{\xi}^{r}-x^\star\|^2 \\
&\le -  \frac{1}{n\tau}\sum_{i=1}^n\sum_{t=0}^{\tau-1} \bE(f_i^{r,t}(x^r)-f^r(x^{\star}))\\
&+\left( \teta \left(1+\frac{\theta}{\teta}\right) + \frac{1}{2\mu}\right)2L^2\eta^2(2\tau^2 B_g^2 +{8 \max_{r,t}\|b^{r,t}\|^2 d nV^2})\\
&
+ \frac{1}{\theta}\bE\| {\frac{\teta}{\tau}} b^{r-1,\tau-1} \bxi\|^2+ \frac{2c_{{\rm VII}}}{\teta} \frac{1}{n\tau} \sum_{i=1}^n\sum_{t=0}^{\tau-1} 2L \bE{( f_i^{r,t}(x^{\star}) -(f^{r})^{\star}  )}.
\end{aligned}       
\end{equation*}
Consequently, we obtain the bound
\begin{equation*}
\begin{aligned}
&\frac{1}{n\tau}\sum_{i=1}^n\sum_{t=0}^{\tau-1} \bE(f_i^{r,t}(x^r)-f^r(x^{\star}))\\
\le \ &  \frac{1}{\teta}\bE\|x_{\xi}^{r+1}-x^\star\|^2 -\frac{1}{\teta}\bE\|x_{\xi}^{r}-x^\star\|^2 +\mathcal{O}\left( \frac{\teta^2}{\eta_g^2 }B_g^2+ \theta \|x^{\star}-x_r^{\star}\|^2\right. \\
&\left.+{\frac{\teta^2}{\eta_g^2 \tau^2} \max_{r,t}\|b^{r,t}\|^2 dn V^2 } + \frac{\teta^2}{\theta} {\|{\frac{1}{\tau}}b^{r-1,\tau-1}\bxi\|^2}
\right).
\end{aligned}    
\end{equation*}
Repeated application of the above inequality yields 
\begin{equation*}
\begin{aligned}
&\frac{1}{R} \sum_{r=0}^{R-1} \frac{1}{n\tau}\sum_{i=1}^n\sum_{t=0}^{\tau-1} \bE(f_i^{r,t}(x^r)-f^r(x^{\star}))\\
\le&\mathcal{O}\left(\frac{1}{ \teta R}+ \frac{\teta^2}{\eta_g^2 }B_g^2+ \left(1+\frac{n}{\eta_g^2}\right) \frac{\teta^2}{\theta}   (\ln (R\tau))^2 \frac{d B_g^2}{n} \frac{(\ln \frac{1}{\delta})}{{\tau^2}\epsilon^2}     + \theta \frac{C_R}{R} \right).    
\end{aligned}    
\end{equation*} 
From the last two terms of the regret upper bound, we observe a trade-off on $\theta$. To balance this trade-off and satisfy condition \eqref{eqtheta}, we set $ \frac{\teta^2}{\theta} (\ln (R\tau))^2=\theta $ and obtain
\begin{equation*}
\theta=\teta\ln({R\tau})\quad\text{and}\quad\teta\le \frac{1}{10L \ln(R\tau)},
\end{equation*}
which yields the bound
\begin{equation*}
\begin{aligned}
&\frac{1}{R} \sum_{r=0}^{R-1} \frac{1}{n\tau}\sum_{i=1}^n\sum_{t=0}^{\tau-1} \bE(f_i^{r,t}(x^r)-f^r(x^{\star}))\\
\le&\mathcal{O}\left(\frac{1}{\teta R} +\frac{\teta^2}{\eta_g^2 }B_g^2+ \left(1+\frac{n}{\eta_g^2}\right)\teta \ln({R\tau}) \frac{d B_g^2(\ln \frac{1}{\delta})}{n {\tau^2}\epsilon^2}+ \teta \ln({R\tau})\frac{C_R}{R}
\right).
\end{aligned}    
\end{equation*}
Since $\teta= \mathcal{O}((R\ln(R\tau)) ^{-\frac{1}{2}})$, it follows that 
\begin{equation*}
\begin{aligned}
&\frac{1}{R} \sum_{r=0}^{R-1} \frac{1}{n\tau}\sum_{i=1}^n\sum_{t=0}^{\tau-1} \bE(f_i^{r,t}(x^r)-f^r(x^{\star}))\\
\le&\mathcal{O}\left(\frac{(\ln(R\tau))^{\frac{1}{2}}}{R^{\frac{1}{2}}} +\frac{1}{R(\ln(R\tau))}\frac{B_g^2}{\eta_g^2 } \right.\\
&\left.+ \left(1+\frac{n}{\eta_g^2}\right)\frac{(\ln(R\tau))^{\frac{1}{2}}}{R^{\frac{1}{2}}}\frac{d B_g^2(\ln \frac{1}{\delta})}{{\tau^2}n \epsilon^2} + \frac{(\ln(R\tau))^{\frac{1}{2}}}{R^{\frac{1}{2}}} \frac{C_{R}}{R}
\right),
\end{aligned}    
\end{equation*}
which completes the proof of Corollary~\ref{coro}.

\end{document}